\documentclass[letterpaper]{article} 
\usepackage{aaai23}  
\usepackage{times}  
\usepackage{helvet}  
\usepackage{courier}  
\usepackage[hyphens]{url}  
\usepackage{graphicx} 
\usepackage{natbib}  
\usepackage{caption} 
\usepackage{algorithm}
\usepackage{algorithmicx}
\usepackage[noend]{algpseudocode}

\usepackage{newfloat}
\usepackage{listings}
\DeclareCaptionStyle{ruled}{labelfont=normalfont,labelsep=colon,strut=off} 
\floatstyle{ruled}
\newfloat{listing}{tb}{lst}{}
\floatname{listing}{Listing}
\usepackage[colorlinks=true,linkcolor=darkblue,citecolor=darkblue, urlcolor=black]{hyperref}
\usepackage[utf8]{inputenc} 
\usepackage{booktabs}       
\usepackage{amsfonts}       
\usepackage{nicefrac}       
\usepackage{microtype}      
\usepackage{xcolor}         
\usepackage{mathtools} 
\usepackage{tikz} 
\usepackage{tasks}
\usepackage{amsmath}
\usepackage{amssymb}
\usepackage{amsthm}
\usepackage{verbatim}
\usepackage{latexsym}
\usepackage{subfigure}
\usepackage{upgreek}
\usepackage[inline]{enumitem}
\usepackage{lipsum}  
\usepackage{dsfont}
\usepackage{multirow}
\usepackage{wasysym}
\usepackage{siunitx}
\usepackage{fontawesome5}

\theoremstyle{plain}
\newtheorem{theorem}{Theorem}[section]

\newtheorem{lemma}[theorem]{Lemma}

\theoremstyle{definition}

\theoremstyle{remark}
\newtheorem{remark}[theorem]{Remark}
\newenvironment{proofsketch}{\proof}{\endproof}
\newtheorem{case}{Case}
\counterwithin{case}{theorem}
\usepackage{thmtools}
\usepackage{thm-restate}

\definecolor{NavyBlue}{HTML}{5f9ea0}
\definecolor{BrickRed}{HTML}{B73239}
\definecolor{darkgreen}{HTML}{046307}
\definecolor{darkteal}{HTML}{2f6882}
\definecolor{Lavender}{HTML}{785d98}
\definecolor{darkblue}{rgb}{0, 0, 0.5}

\newcommand{\banditName}[1]{networked restless bandits}

\newcommand{\alg}[0]{\textsc{Greta}}

\newcommand{\grAgnostic}{$^{\oslash}$}
\newcommand{\grAware}{$^{\dagger}$}
\DeclareMathOperator*{\argmax}{arg\,max}

\DeclareMathOperator*{\dprime}{\prime \prime}
\newcommand{\algorithmicbreak}{\textbf{break}}

\algnewcommand{\algC}[1]{{\color{gray} \Comment{#1}}}
\algnewcommand{\algLC}[2]{{\color{gray} \Statex \hspace{#1em} \(\triangleright\) #2}}

\DeclareMathOperator{\dom}{dom}
\newcommand{\calN}{{\cal N}}
\newcommand{\calNin}{{\cal N}_{\text{in}}}
\newcommand{\calNout}{{\cal N}_{\text{out}}}

\algnewcommand{\IfElse}[2]{
  \algorithmicif\ #1\ \algorithmicelse\ #2}
  
\algnewcommand{\IfThenElse}[3]{
  \algorithmicif\ #1\ \algorithmicthen\ #2\ \algorithmicelse\ #3}

\newcommand*{\email}[1]{\href{mailto:#1}{\texttt{#1}}}

\title{Networked Restless Bandits with Positive Externalities}
\author {
    Christine Herlihy,
    John P. Dickerson\\ 
}
\affiliations {
    Department of Computer Science\\
    University of Maryland, College Park\\
    College Park, MD, USA\\
    \email{cherlihy@umd.edu}, \email{johnd@umd.edu}
}

\begin{document}

\maketitle

\begin{abstract}
Restless multi-armed bandits are often used to model budget-constrained resource allocation tasks where receipt of the resource is associated with an increased probability of a favorable state transition. Prior work assumes that individual arms only benefit if they receive the resource directly. However, many allocation tasks occur within communities and can be characterized by positive externalities that allow arms to derive partial benefit when their neighbor(s) receive the resource. We thus introduce \emph{\banditName{}}, a novel multi-armed bandit setting in which arms are both restless and embedded within a directed graph. We then present \alg{}, a graph-aware, Whittle index-based heuristic algorithm that can be used to efficiently construct a constrained reward-maximizing action vector at each timestep. Our empirical results demonstrate that \alg{} outperforms comparison policies across a range of hyperparameter values and graph topologies. 
For reproducibility purposes, all code is available at: \href{https://github.com/crherlihy/networked_restless_bandits}{\faGithub{} \texttt{crherlihy/networked\_restless\_bandits}}.
\end{abstract}

\section{Introduction}
\label{sec:introduction}
We study the planning task of allocating budget-constrained indivisible resources so as to maximize the expected amount of time that members of a cohort will spend in a desirable state (e.g., adherent to a prescribed exercise regimen). Restless multi-arm bandits (RMABs) are well-suited for such tasks, as they represent each individual as a Markov decision process (MDP) whose stochastic state transitions are governed by an action-dependent transition function. 

Conventionally, an arm must receive the resource at time $t$ to derive any benefit from it, where benefit takes the form of an increased probability of transitioning to the desirable state at time $t+1$ (i.e., relative to non-receipt). However, many resource allocation tasks \emph{occur within communities} and can be characterized by \emph{positive externalities} that allow arms to derive partial, indirect benefit when their neighbor(s) receive the resource. We consider chronic disease management programs as a motivating example. These programs often combine resource-constrained physician support with less cost-intensive, more scalable peer support to encourage participants to make lifestyle modifications. To this end, we introduce \emph{\banditName{}}, a novel multi-armed bandit setting in which arms are both restless and embedded within a directed graph. We then present a graph-aware, Whittle-based heuristic algorithm that is constrained reward-maximizing in this setting.
Our core contributions include:
\begin{enumerate}[label=\textbf{(\roman*)},itemsep=0pt, leftmargin=2em]
    \item Our networked restless bandit model, which lets us represent topological relationships between arms, and associate arm $i$'s receipt of a pull with positive externalities for its neighbors.
    \item \alg, a graph-aware, Whittle index-based heuristic algorithm that lets us efficiently construct a constrained reward-maximizing
    mapping from arms to actions at each timestep. 
    \item Empirical results which demonstrate that \alg{} outperforms comparison policies across a range of hyperparameter values and graph topologies.
\end{enumerate}

\subsection{Related Work}
\label{sec:related}
\par \textbf{Restless bandits}: The restless multi-armed bandit (RMAB) framework was introduced by \citet{whittle1988restless} as a way to model the sequential allocation of a budget-constrained, indivisible resource over a population of $N$ dynamic arms, where: (1)~at most $k \ll n$ arms can receive the resource (i.e., a pull) at any given timestep; and (2) the state of each arm evolves over time, regardless of whether or not it is pulled. We provide a formal description in Section~\ref{sec:restless}. 

\par \textbf{Indexability}: In the general case, it is PSPACE-hard to pre-compute the optimal policy for a given cohort of restless arms~\citep{papadimitriou1994complexity}. However, as conjectured by \citet{whittle1988restless} and proven by \citet{weber1990index}, when each arm is indexable, a tractable solution exists that is provably asymptotically optimal: we can decouple the arms and consider a Lagrangian relaxation of the original problem. In this context, the Whittle index can be thought of as the infimum subsidy required to make an arm indifferent between a pull and passivity, given its current state. Whittle-index based policies use these index values to rank arms when selecting which $k$ arms to pull. 

Proving indexability can be difficult and often requires the problem instance to satisfy specific structural properties, such as the optimality of threshold policies~\citep{liu2010indexability}. Additionally, much of the foundational work in this space focuses on the two-action setting, and cannot be directly extended to the multi-action setting that we consider. 

\citet{glazebrook2011general} do consider the multi-action setting, but for divisible rather than indivisible resources; they also require an arm to consume this resource at a level that is decreasing in the resource charge. \citet{killian2021beyond} study multi-action restless bandits and do not make any of the structural assumptions required to verify indexability, but instead develop a Lagrangian bound-minimization approach; however, they do not consider relationships among arms. 

\citet{mate2020collapsing} introduce the collapsing bandit model, and demonstrate that this problem is indexable when forward or reverse threshold policies are optimal. They also introduce an efficient, closed-form approach to computing the Whittle index called \textsc{Threshold Whittle (TW)}, and empirically demonstrate that this approach performs well even when optimality conditions are not satisfied. We leverage \textsc{TW} as a subroutine within \alg{}.

\par \textbf{Bandits and graphs}:
Prior work at the intersection of multi-armed bandits and graphs has tended to focus on \emph{stochastic}, rather than restless arms, and on \emph{graph-structured feedback} (i.e., rewards), rather than the embedding of arms within a directed graph, and/or the spillover effects associated with allocation in the face of adjacency. For example, \citet{valko2016bandits} examines a graph structure among {\it actions} in \emph{stochastic} bandits, and \citet{lu2021stochastic} examines a graph structure over {\it rewards}. However, we examine a graph structure among {\it arms} in the restless bandit setting.

In recent work,~\citet{ou2022networked} look at a mobile intervention setting. Similarly to our model, they combine the traditional restless bandit setting with network externalities; however, their model and goal are fundamentally different.  Their arms represent locations on a network with pulls impacting a mixture of subpopulations that are located at or near that pull, probabilistically. In contrast, in our model, vertices represent individual arms, and our algorithm exploits---when advantageous---the propensity for allocating a high-cost, high-benefit resource to one arm to unlock potential lower-cost, intermediate-benefit resources for the arm's neighbors.

\section{Model Preliminaries} 
\label{sec:model} 
\subsection{Restless Multi-arm Bandits}
\label{sec:restless}
The restless bandit (RMAB) setting features an
agent with $n \in \mathbb{N}$ arms. The state of each arm evolves over time and in response to the agent's actions, in a way that is governed by the arm's underlying Markov decision process (MDP). Each MDP is defined by a state space, $\mathcal{S}$, an action space, $\mathcal{A}$, a cost function, $C: \mathcal{A} \rightarrow \mathbb{R}$, a local reward function, $r: \mathcal{S} \rightarrow \mathbb{R}$, and a transition function, $P: \mathcal{S} \times \mathcal{A} \rightarrow \mathcal{S}$. The 
objective is to find a policy, $\pi: \mathcal{S} \rightarrow \mathcal{A}$, that maximizes total expected discounted reward over a finite time horizon, $T$---i.e., $\pi^* = \argmax_{\pi} E_{\pi}\left[  R(\cdot)\right]$. The agent must select exactly one action per arm at each timestep, and the associated costs must not exceed the per-timestep budget,~$B~\in~\mathbb{R}_{\geq 0}$.

\subsection{Motivating Example}
For ease of exposition, we ground our networked restless bandit model in a \emph{motivating example}: let arms represent patients striving to adhere to a chronic disease management program, such as an exercise regimen. A patient's ``state'' on any given day is thus determined by whether they adhere (i.e., exercise), or fail to adhere to their regimen. 
To encourage adherence, many such programs feature a combination of resource-constrained physician- and peer support~\citep{fisher2017peer}. Examples include, but need not be limited to, a reminder call from a physician, a supportive message from a fellow participant, or the provision of awareness-raising outreach materials. Thus, a coordinator seeking to maximize the number of patients who exercise over the program's duration might select a small subset of patients each day to receive a call from a physician, and ask these people to message a handful of their peers in turn, or pass along an educational pamphlet to their caregiver(s). In each case, the lower-cost, easier-to-scale information dissemination option amplifies physician outreach, allowing a broader subset of individuals to receive partial benefit.

\subsection{Networked Restless Bandits}
\label{sec:networkedRmab}
With this motivating example in mind, we now introduce our networked restless bandit model, which allows us to model directed relationships among arms. Given a set of $n$ arms, let $G=(V,E)$ be a directed graph, and let there exist a bijective mapping $\upvarphi$ from arms to vertices --- i.e., $\forall v \in V,\ \exists! \ i \in [n] \text{ s.t. } \upvarphi(i) = v$. Let a directed edge, $e \in E$, exist between arms $u$ and $v$ if it is possible for $v$ to benefit indirectly when $u$ receives a pull. Let~$\calNin(u)~=~\{v\in V \ |~\ \exists e_{v,u} \in E\}$ and $\calN_{out}(u)~=~\{v\in V \ |~\ \exists e_{u,v} \in E\}$ represent $u$'s one-hop indegree and outdegree neighborhoods, respectively. This graph is assumed to be constructed or operated by the agent; as such, it is assumed to be observable. Real-life examples with this property include mentoring programs and online social networks.

\par \textbf{State space}: We consider a discrete state space, $\mathcal{S}~:=~\{0,1\}$, 
where the states admit a total ordering by desirability, such that state 1 is more ``desirable'' than state 0. In our example, state 0 represents non-adherence to the exercise regimen, while state 1 represents adherence. We assume 
each arm's state is observable (e.g., via fitness tracker data).

\par \textbf{Action space}: The traditional restless bandit setting considers a binary action space, $\mathcal{A} := \{0,1\}$, where 1 (or 0) represents the decision to pull (or not pull) arm $i$ at time~$t$. To model positive externalities, we define an extended action space, $\mathcal{A} := \{0:  \texttt{no-act}, 1: \texttt{message}, 2: \texttt{pull}\}$. Here, actions 0 and 2 correspond to the actions \emph{don't pull} and \emph{pull} respectively. We note that our \texttt{message} action need not represent a literal message. Instead, it represents an intermediate action with respect to desirability that gets ``unlocked'' as an available action for vertex $v$ at time $t$ only when some vertex $u \in \calNin(v)$ receives a pull at time $t$. 

\par \textbf{Transition function}:
For each arm $i \in [n]$, let $P^{a,i}_{s,s^\prime}$ represent the probability that arm $i$ will transition from state $s$ to ${s}^\prime$ given action $a$. In the offline setting, these transition matrices are assumed to be static and known to the agent at planning time. This assumption is reasonable when historical data from the same or similar population(s) provides a source for informative priors, as is common in many domains, including healthcare and finance \citep{steimle2017markov, pasanisi2012estimating}. Extension to the online setting where transition matrices must be learned is possible via Thompson sampling~\citep{thompson1933likelihood, ortner2012regret, jung2019regret, jung2019thompson}.

We assume nonzero transition matrix entries, and impose two sets of domain-motivated \textbf{structural constraints}~\citep{mate2020collapsing}: 
\begin{enumerate*}[label=(\roman*)]
    \item  $\forall a \in \mathcal{A}, P^a_{0,1}< P^a_{1,1}$ and 
    \item $\forall (a, a^\prime) \in \mathcal{A} \times \mathcal{A}, a < a^\prime \rightarrow P^a_{0,1} < P^{a^\prime}_{0,1}; P^a_{1,1} < P^{a^\prime}_{1,1}$.
\end{enumerate*}
Constraint set (i) implies that each arm is more likely to stay in the desirable state (i.e., $s=1$) than transition there from the undesirable state (i.e., $s=0$). Constraint set (ii) implies that messages and pulls are beneficial when received and that a strict preference relation over actions can be defined for each arm, such that \texttt{no-act} $\prec$ \texttt{message} $\prec$ \texttt{pull}.

\par \textbf{Cost function}: We map our action space to the cost vector $\vec{c} = [0, \psi, 1]$, where $0 \leq \psi < 1$. Intuitively, this mapping preserves standard notion that no cost is incurred when an arm does not receive any form of intervention. It also encodes the idea that the more beneficial an action is, the more expensive it is to provide, which motivates us to exploit positive externalities. 
Additionally, when there are no edges, i.e., $E = \emptyset$, and no messages can be sent, the unit cost of a pull lets us recover the original restless bandit setting, where we must choose which $k \ll n$ arms to pull at each timestep. 

\par \textbf{Objective and constraints}: It is possible, though not tractable at scale, to take a constrained optimization-based approach to solving for the optimal policy, $\pi^*$. We build on \citet{killian2021beyond}'s approach below to show how our constrained setting can be modeled. To begin, let $\mathbf{s}$ represent a vector containing the state of each arm, i.e. $[s^i \in \mathcal{S}|i\in[n]]$, and let $\mathbf{X}$ represent a matrix containing binary decision variables, one for each of $n$ arms and $|{\cal A}|$ actions. We require our local reward function, $r: \mathcal{S} \rightarrow \mathbb{R}$ to be non-decreasing in $s$, which is consistent with our goal of maximizing the expected time that each arm spends in the ``desirable'' state. Equation~\ref{eqn:objective} formalizes our task:

{\centering
\begin{equation}
\label{eqn:objective}
\begin{array}{l@{}ll}
J(\mathbf{s}) = \max\limits_{\mathbf{X}} & \left\{ \displaystyle\sum_{i=0}^{n-1} r^i(s^i) + \beta \mathbb{E}[J(\mathbf{s}^\prime), \mathbf{X}]\right\}  & \\

\text{subject to } & \displaystyle\sum_{i=0}^{n-1} \sum_{j=0}^{|\mathcal{A}|-1} x_{i,j} \cdot c_j \leq B & \\

& x_{i,1} \leq \displaystyle\bigvee_{i^\prime \in \calNin(i)}
x_{i^\prime, 2} & \forall i \in [n] \\ 

& \displaystyle\sum_{j = 0}^{|\mathcal{A}|-1} x_{i,j} = 1 & \forall i \in [n]\\

& \mathbf{X} \in \{0,1\}^{n \times |\mathcal{A}|} &

\end{array}
\end{equation}}

Our goal is to find assignments of the decision variables contained in $\mathbf{X}$ such that expected discounted reward is maximized, subject to a series of feasibility constraints: (i) across all actions and arms, do not expend more than $B$ budget; (ii) ensure that if \texttt{message} is chosen for an arm $i$, then that arm has at least one indegree neighbor $i'$ such that \texttt{pull} was chosen; and, (iii) ensure that each arm receives exactly one action at each timestep. However, two challenges arise: (1) a direct solution via value iteration is exponential in $n$, and (2) Lagrangian relaxation-based approaches rely on the decoupling of arms, which jeopardizes the satisfaction of our neighborhood constraint on actions. This motivates us to propose a graph-aware, Whittle-based heuristic algorithm.

\section{Algorithmic Approach}
\label{sec:algorithm}
Here, we introduce \alg{}, a graph-aware, Whittle-index-based heuristic algorithm that can be used to efficiently construct a constrained reward-maximizing policy. A key insight that \alg{} exploits is that while we cannot decouple arms in the networked setting, since we must know whether any of an arm's indegree neighbors will receive a pull at time $t$ to know if the arm is eligible to receive a message, we \emph{can} compute two sets of Whittle indices for each arm, by considering each active action as a separate instance of a two-action problem. Note that the structural constraints ensure that for a given state, an arm will require a higher subsidy to forgo a pull as opposed to a message. We can then construct an augmented graph that allows us to compare the cumulative subsidy required for the arms represented by directed edge $(u,v)$ to forgo a \emph{pull} and \emph{message}, respectively to those required by other directed edges $\in G$ (including, importantly, the inverse action-pair implied by edge $(v,u)$).  

\subsection{\textsc{Greta}: A Graph-aware Heuristic}
\label{sec:algDetails}
\par \textbf{Set-up}: We begin by building an augmented graph, $G^\prime$. This graph contains every vertex and edge in $G$, along with a dummy vertex, $-1$, and directed edge $(u,-1) ~\forall u \in V$. This lets us map each directed edge $(u,v)$ in $G$ to the action pair $(\texttt{pull}, ~\texttt{message})$, and $(u,-1)$ to $(\texttt{pull}, ~\texttt{no-act})$. We also construct an augmented arm set, $[n] \cup \{-1\}$, and extend our bijective mapping from arms to vertices such that $\upvarphi: -1 \mapsto -1$. Appendix~\ref{app:constructGprime} provides pseudocode.

Next, we pre-compute the Whittle index for each vertex-active action combination $(v,\alpha) \in V^\prime \times \mathcal{A} \setminus \{0\}$. When we compute the Whittle index for a given $(v, \alpha)$ pair, we seek the infimum subsidy, $m$, required to make arm $i$ (i.e., $\upvarphi^{-1}(v)$) indifferent between passivity (i.e., $\texttt{no-act}$) and receipt of action $\alpha$ at time $t$~\citep{whittle1988restless}.
We cannot compute the Whittle index for our placeholder $-1$ vertex because it is not attached to an MDP, so we map it to 0.
\begin{algorithm}[H]
\caption{Compute Whittle indices for $V^\prime \times \mathcal{A} \setminus \{0\}$}
\label{alg:whittle}
\footnotesize 
\begin{algorithmic}[1] 
\Procedure{\textsc{Whittle}}{$V^\prime, \alpha \in \{1,2\}, \upvarphi$}
\State $\lambda := i \mapsto
  \begin{cases}
    0, & \text{if } i = -1\\ 
    \inf_m \{m \mid V_m(s_t^i, a_t^i = 0) \geq \\\hspace{1.5cm} V_m(s_t^i, a_t^i = \alpha)\}, & \text{otherwise} \\
  \end{cases}$
\State\Return$W_\alpha \gets \{\lambda \circ \upvarphi^{-1}(v) \ | \ v \in V^\prime \}$
\EndProcedure
\end{algorithmic}
\end{algorithm} 

{\small
\begin{equation}
\begin{aligned}
    \label{eqn:valueFunc}
    V_m(s_t^i) = 
    \max \begin{cases}
    m + r(s_t^i) + \beta V_m \left(s_{t+1}^i\right) & \textit{no-act} \\
    \begin{array}{r@{}}
        r(s_t^i)+\beta [s_t^i V_m\left(P^{\alpha}_{1,1}\right) + \\
        (1-s_t^i) V_m\left(P^{\alpha}_{0,1}\right)]  
      \end{array} & \textit{$\alpha$}
    \end{cases}
    \end{aligned}
\end{equation}}
The value function represents the maximum expected discounted reward that arm $i \in [n]$ with state $s_t^i$ can receive at time $t$ given a subsidy $m$, discount rate $\beta$, and active action $\alpha \in \{1,2\}$.

\textbf{\alg}: With our augmented graph and Whittle index values in hand, we now present our algorithm. We provide pseudocode 
 in Algorithm~\ref{alg:heuristicClean}, and structure our exposition sequentially. At each timestep $t \in T$, \alg{} takes as \emph{input}: (1) an augmented set of restless arms, $[n] \cup \{-1\}$ embedded in an augmented directed graph, $G^\prime = (V^\prime, E^\prime)$; (2) a budget, $B \in \mathbb{R}$; (3) a cost function, $C: \mathcal{A} \rightarrow \mathbb{R}$; (4) a message cost, $\psi \in [0,1)$; and (5) a set of Whittle index values per active action $\alpha \in \{1,2\}$, denoted by $W_1$ and $W_2$, respectively. 
Given these inputs, \alg{} \emph{returns} a reward-maximizing, constraint-satisfying action vector, $\vec{a}_t$.

\begin{algorithm}[!h]
\caption{\textsc{Greta}: graph-aware, Whittle-based heuristic\\ Note: all sorts are descending; arrays are zero-indexed.}
\label{alg:heuristicClean}
\footnotesize
\begin{algorithmic}[1] 
\Procedure{\textsc{Greta}}{$G^\prime, V^\prime, E^\prime, B, C, \psi, W_1, W_2$}
\State $\vec{a}_t \gets 0^{|V|}$ 
\State $B^\prime \gets B$ 
\While {$\vee_{e \in E^\prime}$ \Call{GetCost}{$u,v,\vec{a}_t,C$} $\leq B^\prime \land E^\prime \neq \emptyset$}
\State \texttt{b} $\gets \min(B^\prime,2)$
\algLC{3}{Consider only pulls}
\State $\hat{a}_{2}, \nu_2 \gets$ \Call{PullOnly}{$E^\prime,\lfloor{\texttt{b}\rfloor}, W_2$}
\algLC{3}{Consider pulls \emph{and} messages}
\State $\hat{a}_{(1,2)}, \nu_{(1,2)}, E^\prime_{\oslash} \gets$ \Call{MP}{$G^\prime, \texttt{b}, C, \psi, \vec{a}_t, W_1, W_2$}
\algLC{3}{Select max-val candidate actions; update $\vec{a}_t, B^\prime, G^\prime$}
\If {$\nu_2 \geq \nu_{(1,2)}$}
\State $\vec{a}_t, B^\prime \gets$ \Call{ModActsB}{$G^\prime, C, \hat{a}_{2}, \vec{a}_t, B^\prime$}
\State $E^{\prime}, G^{\prime} \gets$ \Call{\textsc{UpdateG}}{$V^{\prime}, E^{\prime}, \hat{a}_{2}, \emptyset$}
\Else 
\State $\vec{a}_t, B^\prime,  \gets$ \Call{ModActsB}{$G^\prime, C, \hat{a}_{(1,2)}, \vec{a}_t, B^\prime$}
\State $E^{\prime}, G^{\prime} \gets$ \Call{\textsc{UpdateG}}{$V^{\prime}, E^{\prime}, \hat{a}_{(1,2)}, E^\prime_{\oslash}$}
\EndIf 
\EndWhile 
\State\Return $\vec{a}_t$
\EndProcedure
\end{algorithmic}
\end{algorithm}
In lines 2-3 of Algorithm~\ref{alg:heuristicClean}, we 
initialize $\vec{a}_t$ such that each vertex is mapped to 0 (\texttt{no-act}), and set our \emph{remaining budget} variable, $B^\prime$, equal to the per-timestep budget, $B$. 

In lines {4-13}, we iteratively update our action vector  $\vec{a}_t$ until 
we have insufficient remaining budget to afford any available edge-action pair, or our augmented edge set, $E^\prime=~\emptyset$. The termination check in line {4} requires us to: (1) check if we've already incurred the cost of a \emph{pull} or \emph{message} (\emph{message}) for vertex $u$ ($v$); and (2) offset accordingly when we compute the cost of $(a_t^u = 2, a_t^v = 1)$, per Alg.~\ref{alg:getCost}.

\begin{algorithm}[!h]
\caption{Compute cost to pull $u$ and message $v$}
\label{alg:getCost}
\footnotesize
\begin{algorithmic}[1] 
\Procedure{\textsc{GetCost}}{$u, v, \vec{a}_t, C$}
\State $c_{u} \gets C(2)(1-\mathds{1}(a^u_t > 0)) + \mathds{1}(a^u_t = 1)(C(2) - C(1))$
\State $c_{v} \gets C(1)(1 - \mathds{1}(a^v_t =1 \lor v = -1))$\\
\Return $c_u + c_v$
\EndProcedure
\end{algorithmic}
\end{algorithm}
The subroutines called in lines 6-7 of \alg{} serve to ensure that we will only deviate from the pull-assignment choices of graph-agnostic \textsc{Threshold Whittle}---i.e., by considering a combination of pulls \emph{and} messages---when it is strictly beneficial to do so. 

Since pulls have unit cost, and $\psi \in [0,1)$, we consider our per-timestep budget in sequential chunks of 2. We have two options for allocating each chunk over actions: (1) considering \emph{only} pulls, and selecting the two arms with highest $W_2$ index values; or (2) considering messages \emph{and} pulls, and selecting the set of directed $(u,v)$ edges with highest edge-level subsidies such that each $u$ receives a pull, and each $v$ (excluding $-1$) receives a message. In lines 8-13, we select the candidate action set with the highest cumulative subsidy, and update $\vec{a}_t, B^\prime, \text{ and } G^\prime$ accordingly.

\par \textbf{Pulls only}: Allocation option (1) maps arms who have yet to receive a pull at time $t$ to candidate actions $\in \{0,2\}$ by sorting their $W_2$ index values in descending order and selecting the top-2 arms to receive pulls. App.~\ref{app:pseudocode} gives pseudocode (Alg~\ref{alg:pullsOnly}).

\par \textbf{Messages \emph{and} pulls}: Allocation option (2) maps arms to candidate actions by computing an edge index value for each directed edge $\in E^\prime$. Algorithm~\ref{alg:msgPull} provides pseudocode.
\begin{algorithm}[H]
\caption{Cumulative subsidy of max pull-message set\\ Note: all sorts are descending; arrays are zero-indexed.}
\label{alg:msgPull}
\footnotesize 
\begin{algorithmic}[1] 
\Procedure{\textsc{MP}}{$G^\prime,\texttt{b} \in \mathbb{R}, C,\psi, \vec{a}_t, W_1, W_2$}
\State $G^{\dprime} = (V^{\dprime}, E^{\dprime}) \gets G^\prime$ 
\State $\hat{a}_{(1,2)}: v \in V^{\dprime} \mapsto \vec{a}_t^v$ 
\State $f: (u,v) \in E^{\dprime} \mapsto \mathbb{R}$
\State $E^\prime_{\oslash} \gets \emptyset$
\State $\nu_{(1,2)} \gets 0$
\While {$\vee_{e \in E^{\dprime}}$ \Call{GetCost}{$u,v,\hat{a},C$} $\leq \texttt{b} \land E^{\dprime} \neq \emptyset$}
\For {$u \in V^{\dprime} \setminus \{-1\}$}
\State $\calNout^\prime(u) \gets \left\{v | (u,v) \in E^{\dprime} \land \hat{a}_{(1,2)}^v = 0 \right\}$
\State \Call{EdgeIndices}{$f^{\dprime}, u, \calNout^\prime(u), \texttt{b}, \psi, W_1, W_2$}
\EndFor
\State \texttt{values} $\gets \textsc{sort}(\{f((u,v)) | (u,v) \in E^{\dprime}\})$
\If {$|\texttt{values}| = 0$}
\State \algorithmicbreak
\EndIf 
\For {$f((u,v)) \in \texttt{values} $}
\State \texttt{cost}$_{u,v} \gets $ \Call{ComputeCost}{$u,v,\hat{a}_{(1,2)}, C$}
\If {\texttt{cost}$_{u,v} \leq \texttt{b}$}
\State $h : u \mapsto 2; v \mapsto 1$
\State $\hat{a}_{(1,2)}, \texttt{b}  \gets$ \Call{ModActsB}{$G^{\dprime}, C, h, \hat{a}_{(1,2)}, \texttt{b}$}
\State $E^{\dprime}, G^{\dprime} \gets$ \Call{\textsc{UpdateG}}{$V^{\dprime}, E^{\dprime}, \hat{a}_{(1,2)}, \emptyset$}
\State $\nu_{(1,2)} \mathrel{{+}{=}} f((u,v))$
\State $E^\prime_{\oslash} \gets E^\prime_{\oslash} \cup \{(u,v)\}$
\State \algorithmicbreak
\EndIf
\EndFor
\EndWhile 
\algLC{1}{Return best arm-actions, cumulative subsidy, $E^\prime_{\oslash}$}
\State\Return $\hat{a}_{(1,2)}, \nu_{(1,2)}, E^\prime_{\oslash}$
\EndProcedure
\end{algorithmic}
\end{algorithm}
In line {2} of Algorithm~\ref{alg:msgPull}, we start by defining $G^{\dprime}$ to be a local copy of our augmented graph, $G^{\prime}$. We then create a function, $\vec{a}_{(1,2)}$ to map each vertex $v \in V^{\dprime}$ to its candidate action, which we initialize to be $\vec{a}_t^v$ (line {3}). We do this because we require the current $G^\prime$ to determine which $(\texttt{pull}_u,\texttt{message}_v)$ edge-action combinations are possible, and for $\vec{a}_t$ to correctly compute the cost of these hypothetical actions, but we don't want to modify $\vec{a}_t$ or $G^\prime$ in-place. Next, in lines {4-5}, we define a function, $f$ that maps each edge $(u,v) \in E^\prime$ to a real-valued edge index value, and a set, $E^\prime_{\oslash}$, to hold the edges we will need to remove from $G^\prime$ if we select the candidate actions returned by Algorithm~\ref{alg:msgPull}. In line {6}, we initialize $\nu_{(1,2)} = 0$ to represent the cumulative subsidy of our candidate action set.

In lines {7-22} of Algorithm~\ref{alg:msgPull}, we iteratively update our candidate action function, $\hat{a}_{(1,2)}$, until we run out of (small-\texttt{b}) budget, or $E^{\dprime} = \emptyset$. Inside each iteration of the \textsc{while}-loop, we begin by computing an edge index value for each directed edge $(u,v) \in E^\prime$ (lines 8-10). To do this, we loop over vertices in $V^\prime \setminus \{-1\}$ (line 8), and for each vertex $u$, let $\calNout^\prime(u) \subseteq \calNout(u)$ represent the subset of $u$'s one-hop out-degree neighbors currently slated to receive a \texttt{no-act} at time $t$. 

For each edge $(u,v) \in \calNout^\prime(u)$, our edge index value represents the cumulative subsidy required to forgo a \emph{pull} for arm $u$ (i.e., $W_2^u$) and a \emph{message} for arm $v$ (i.e., $W_1^v$). Note: if we pull $u$, message $v$, and have budget left over, we can message up to $|\mathcal{M}^u_t|$ vertices $v^\prime \in \calNout^\prime(u)$ at time $t$ \emph{without} incurring additional pull costs, where $|\mathcal{M}^u_t|=|\calNout^\prime(u)|$ if $\psi=0$ and $ \min(\lfloor{\nicefrac{\texttt{b}}{\psi}\rfloor}, |\calNout^\prime(u)|) \text{ for } \psi \in (0,1)$.

To exploit this diminishing marginal cost, we sort $u$'s, neighbors by their index-values and let the max-valued edge represent the cumulative, cost-feasible value of $\calNout^{\prime}(u)$, rather than just $(u,v)$. Algorithm~\ref{alg:computeIndices} provides pseudocode.
\begin{algorithm}[H]
\caption{Compute edge index values \\ Note: all sorts are descending; arrays are zero-indexed.}
\label{alg:computeIndices}
\footnotesize
\begin{algorithmic}[1] 
\Procedure{\textsc{EdgeIndices}}{$f, u, \calNout^{\prime}(u), \texttt{b}, \psi, W_2, W_1$}
\State \texttt{n\_msgs}$\gets |\calNout^{\prime}(u)|$ \IfElse{$\psi = 0$}{$\min(\lfloor{\nicefrac{\texttt{b}}{\psi}\rfloor}, |\calNout^{\prime}(u)|)$}
\State $\texttt{msg\_values} \gets \textsc{sort}(g: v \in \calNout^{\prime}(u) \mapsto W_1^v)$
\State \texttt{max\_edge} $\gets (u, \argmax_v \texttt{msg\_values})$
\For {$v \in \calNout^{\prime}(u)$}
\If {$(u,v) = $\texttt{ max\_edge}}
\State $f((u,v)) \gets W_2^u + \sum_{i=0}^{\texttt{n\_msgs}-1} \texttt{msg\_values}_i$
\Else\State $f((u,v)) \gets W_2^u + W_1^v$
\EndIf
\EndFor
\State\Return \algC{$f$ is updated in-place}
\EndProcedure
\end{algorithmic}
\end{algorithm}

Then, in lines {11-13} of Algorithm~\ref{alg:msgPull}, we sort the edge-index values in descending order. Note that we break if there are no values to be sorted; this corresponds to the scenario in which no additional pulls are available/cost-feasible, and every arm not receiving a pull is already receiving a message, but we still have budget left---i.e., when $\psi=0$. In lines {14-22}, we choose the top cost-feasible edge-action pair from our sorted list, and update our candidate action function, $\vec{a}_{(1,2)}$ and local budget, $\texttt{b}$ accordingly. Note that if $\psi=0$ and arm $u$ receives a pull, we message every $v \in \calNout^{\prime}(u)$. App.~\ref{app:pseudocode} provides pseudocode for the \textsc{ModActsB} subroutine (see  Algorithm~\ref{alg:updateActionsBudget} in Appendix~\ref{app:updateActionsB}).

Finally, we update our local copy of the augmented graph by removing $(u,v)$, as well as any directed edge terminating in $u$, and the placeholder edge, $(u,-1)$. This is because: (a) we do not want to reconsider the edge-action pair we've selected; and (b) by virtue of how we select $(u,v)$, $f((u,v)) \geq f((\cdot,u))$ \emph{or} any such $u$-terminating edge is cost-prohibitive.
App.~\ref{app:pseudocode} provides pseudocode for the \textsc{UpdateG} subroutine (see Algorithm~\ref{alg:updateGprime}). We conclude the \textsc{MP} subroutine (Algorithm~\ref{alg:msgPull}) by returning our candidate action function, $\hat{a}_{(1,2)}$, the associated cumulative subsidy value, $\nu_{(1,2)}$, and the set of candidate edges to be removed from $G^\prime, E^\prime_{\oslash}$. 

\par \textbf{Putting the pieces together}: With the exposition of each of \alg's subroutines complete, we now return to lines {8-13} of Algorithm~\ref{alg:heuristicClean}. We compare the cumulative subsidy values returned by the \textsc{PullOnly} and \textsc{MP} subroutines, and use the candidate action function associated with the maximum cumulative subsidy to update our action vector, $\vec{a}_t$, remaining budget, $B^\prime$, and augmented graph, $G^\prime$. When the \textsc{while}-loop terminates, we return $\vec{a}_t$. By virtue of how this action vector is constructed, it is reward-maximizing and guaranteed to satisfy the budget constraint. 

\subsection{Theoretical Analysis}
\label{sec:theoretical}
\par \textbf{Bounding expected reward}: Per Theorem~\ref{thm:expectedReward}, the expected cumulative reward of \alg{} with message cost, $\psi >0$, will be lower-bounded by that of graph-agnostic \textsc{Threshold Whittle}, and upper-bounded by \alg{} with $\psi=0$. See Appendix~\ref{app:proofs} for a complete proof. 
\begin{restatable}[]{theorem}{expectedreward}\label{thm:expectedReward}
For a given set of $[n]$ restless or collapsing arms with transition matrices satisfying the structural constraints outlined in Section~\ref{sec:networkedRmab}, corresponding directed graph, $G=(V,E)$, budget $B \in \mathbb{R}_{\geq 0}$, non-decreasing local reward function, $r: \mathcal{S} \rightarrow \mathbb{R}$, cumulative reward function, $R$, and cost vector $\vec{c} = [0,\psi,1]$ such that $\psi \in [0,1)$, we have: $\mathbb{E}_{\textsc{TW}}[R] \leq \mathbb{E}_{\textsc{GH}, \psi > 0}[R] \leq \mathbb{E}_{\textsc{GH}, \psi = 0}[R]$
\end{restatable}

\begin{proofsketch}
The \emph{first inequality} follows from how \textsc{Greta} constructs each $\vec{a}_t$.
The \emph{second inequality} follows from the fact that: (a) per our structural constraints and choice of $r$, $E[r^i_{t}|s^i_t,a^i_t]$ is strictly increasing with $a^i_t \ \forall i, t$; and (b) for $\psi=0$, we can message at least as many arms as when $\psi>0$.
\end{proofsketch}

\par \textbf{Computational complexity}: Per Theorem~\ref{thm:timeComplexity}, \alg{} is efficiently computable in time polynomial in its inputs; see Appendix~\ref{app:proofs} for a complete proof.

\begin{restatable}[]{theorem}{timecomplexity}\label{thm:timeComplexity} 
 For convenience, let: $\xi = \mathds{1}(\psi > 0) \times \min(|E^{\prime}|^2, \lfloor{\frac{B}{\psi}\rfloor}|E^\prime|) + \mathds{1}(\psi = 0) \times~|V^{\prime}||E^\prime|$. 
Then, for $\psi \in [0,1)$ and time horizon, $T$, the time complexity of \textsc{Greta} is:
\small{
\begin{equation*}
\begin{cases}
		O\left(\max\left(\xi^2|V^\prime|^2\log|V^\prime|, \  \xi^2|V^\prime||E^\prime|^2 \right)T \right), \quad \text{if $\psi > 0$}\\
            O\left(\max\left(\xi^2|V^\prime|^2\log|V^\prime|, \ \xi^2|V^\prime||E^\prime|^2, \ \xi^2|V^\prime|^2|E^\prime|  \right)T \right), \\ \quad \text{otherwise}
		 \end{cases}
\end{equation*}}
\end{restatable}

These bounds indicate that \alg{} is well-suited for sparse graphs and values of $\psi = 0$ or $\psi \rightarrow 0.5$ ($\psi > 0.5$ will also improve runtime, but may reduce opportunities to exploit externalities). Conversely, pathological cases will include large-scale dense graphs and values of the message cost, $\psi$, which approach but do not equal 0. We consider improving scale to be a valuable direction for future work. The combinatorial nature of the problem we consider suggests that sampling and/or distributed methods will be critical in this regard~\citep{zhou2020graph, almasan2022deep}.

\section{Experimental Evaluation}
\label{sec:experimental}
In this section, we demonstrate that \alg{} consistently outperforms a set of robust graph-agnostic and graph-aware comparison policies. We begin by identifying the set of policies we compare against, as well as our evaluation metrics, graph generation, and mapping of arms to vertices. We proceed to present results from three experiments: (1) \alg{} versus the optimal policy (for small $n$); (2) \alg{} versus comparison policies for a fixed cohort and graph; and (3) \alg{} evaluated on a series of different budgets, message costs, and graph topologies. 

\subsection{Experimental Setup}
\label{sec:experimentalSetup}

\par \textbf{Policies}:
In our experiments, we compare the policy produced by \alg{} against a subset of the following graph-$\{$agnostic$^{\oslash}$ and aware$^{\dagger}\}$ policies:

\begin{table}[!h]
\centering
\resizebox{\columnwidth}{!}{%
\begin{tabular}{|c|l|}
\hline
\begin{tabular}[c]{@{}c@{}}\textsc{Threshold} \\ \textsc{Whittle}  \textsc{(TW)\grAgnostic}\end{tabular} &
  \begin{tabular}[c]{@{}l@{}}Compute Whittle index values using pull as (only) active \\ action. Pull $\lfloor{B\rfloor}$ arms with highest Whittle index values; \\ all others get no-act  \citep{whittle1988restless, mate2020collapsing}.\end{tabular} \\ \hline
\begin{tabular}[c]{@{}c@{}} \textsc{Random\grAware}\end{tabular} &
  \begin{tabular}[c]{@{}l@{}}Construct $G^\prime$; select budget-feasible edge-action pairs \\ uniformly at random until budget exhausted.\end{tabular} \\ \hline
\begin{tabular}[c]{@{}c@{}} \textsc{Centrality-} \\ \textsc{weighted  Random\grAware}\end{tabular} &
  \begin{tabular}[c]{@{}l@{}}Construct $G^\prime$; select budget-feasible edge-action pairs \\ weighted by out-degree centrality of \texttt{src} vertex 
until \\  budget exhausted.\end{tabular} \\ \hline
\begin{tabular}[c]{@{}c@{}} \textsc{Myopic\grAware}\end{tabular} &
  \begin{tabular}[c]{@{}l@{}}Construct $G^\prime$; sort edge-action pairs  by expected reward \\ at $t+1$. Select cost-feasible pairs until budget exhausted.\end{tabular} \\ \hline
\begin{tabular}[c]{@{}c@{}}
\textsc{Value} \\ \textsc{Iteration (VI)\grAware}\end{tabular} &
  \begin{tabular}[c]{@{}l@{}} 
  Find the optimal policy via value iteration for \emph{system-level} \\ MDP (intractable at scale, but computable for small $|V|$ and $|E|$).\end{tabular} 
  \\ \hline
\end{tabular}%
}
\caption{Comparison policies}
\end{table}

We note that in the restless (but graph-agnostic) setting: (1) \textsc{Random} and \textsc{Myopic} are common baselines. Here, we have extended them to the networked setting. (2) \textsc{Threshold Whittle} represents a state-of-the-art approach. To the best of our knowledge, no additional (efficiently computable) graph-aware policies exist for the novel networked restless bandit setting we propose.

\par \textbf{Objective}: Our optimization task is consistent with assigning equal value to each timestep that any arm spends in the ``desirable'' state. This motivates our choice of a local reward function    $r_t(s_t^t) \coloneqq s_t^i \in \{0,1\}$ and undiscounted cumulative reward function ${R}(r(s))~\coloneqq~\sum_{i \in [N]} \sum_{t\in[T]}~r(s^i_t)$.

\emph{Intervention benefit (IB)}: For each policy, we compute total expected reward, $\mathbb{E}_{\pi}[R(\cdot)]$, by taking the average over 50 simulation seeds. We then compute the intervention benefit as defined in Equation~\ref{eq:ib}, where \textsc{NoAct} represents a policy in which no pulls or messages are executed, and \textsc{GH} represents the policy produced by \alg.
\begin{equation}\label{eq:ib}
\textsc{IB}_{\text{NoAct}, \text{GH}}(\pi) \coloneqq \frac{\mathbb{E}_{\pi}[R_{\pi}(\cdot)] - \mathbb{E}_\text{NoAct}[R(\cdot)]}{\mathbb{E}_\text{GH}[R(\cdot)] - \mathbb{E}_\text{NoAct}[R(\cdot)]}
\end{equation}

\par \textbf{Graph generation}: For each cohort of $n$ restless arms that we consider in our experiments, we use a stochastic block model (SBM) to generate a graph with $|V|=n$ vertices~\citep{holland1983stochastic}. This generator partitions the vertices into blocks and stochastically inserts directed edges, with hyperparameter $p_{in} (p_{out}) \in [0,1]$ controlling the probability that a directed edge will exist between two vertices in the same (different) block(s). 

We consider two options for $\upvarphi: [n] \rightarrow V$: (1) \emph{random}; and (2) \emph{by cluster}. For mapping (1), we generate $\lceil{\frac{n}{10}\rceil}$ blocks of uniform size, and map arms to vertices---and, by extension, to blocks---uniformly at random. This mapping represents allocation settings with a peer support component where participants are randomly assigned to groups, without regard for their behavioral similarity.

For mapping (2), we use an off-the-shelf \textsc{k-means} algorithm to cluster the arms in flattened transition-matrix vector space~\citep{pedregosa2011scikit}. We use the cardinality of the resulting clusters to determine the size of each block and map arms to vertices based on cluster membership. This mapping represents intervention allocation settings with a peer support component where participants with similar transition dynamics are grouped together.

\subsection{\alg{} vs. the Optimal Policy}
\label{sec:gretaVsOptimal}
In this experiment, we compare \alg{} to $\pi^*_{\text{VI}}$, where $\pi^*_{\text{VI}}$ denotes the optimal policy obtained via value iteration for the \emph{system-level} MDP~\citep{sutton2018reinforcement}. This system-level MDP has state space $\mathcal{S}^\prime := \{\mathcal{S}\}^n$, action space $\mathcal{A}^\prime := \{\mathcal{A}\}^n$, a transition function, $P:\mathcal{S}^\prime \times \mathcal{A}^\prime \rightarrow \mathcal{S}^\prime$, and reward function, $R^\prime = \sum_{i \in [n]} s^i$. To ensure budget and neighborhood constraint satisfaction, only cost- and topologically feasible actions, $\mathcal{A}^{\dprime} \subseteq \mathcal{A}^\prime$ are considered. Figure~\ref{fig:gretaVsOpt} reports results for a synthetic cohort of 8 arms embedded in a fully connected graph (i.e., $p_{\text{in}} = p_{\text{out}} = 1.0$). We let $T=120, \psi=0.5$, and report unnormalized $\mathbb{E}_{\pi}[R]$, along with margins of error for 95\% confidence intervals computed
over 50 simulation seeds for values of $B \in \{1,1.5,2,2.5,3\}$. \alg{} outperforms \textsc{TW} for each value of $B$ (with predictably larger gaps for values with remainders $=\psi$ that graph-agnostic \textsc{TW} cannot exploit), and is competitive with respect to $\pi^*_{\text{VI}}$.

\begin{figure}[!h]
\centering
  \includegraphics[width=0.9\columnwidth]{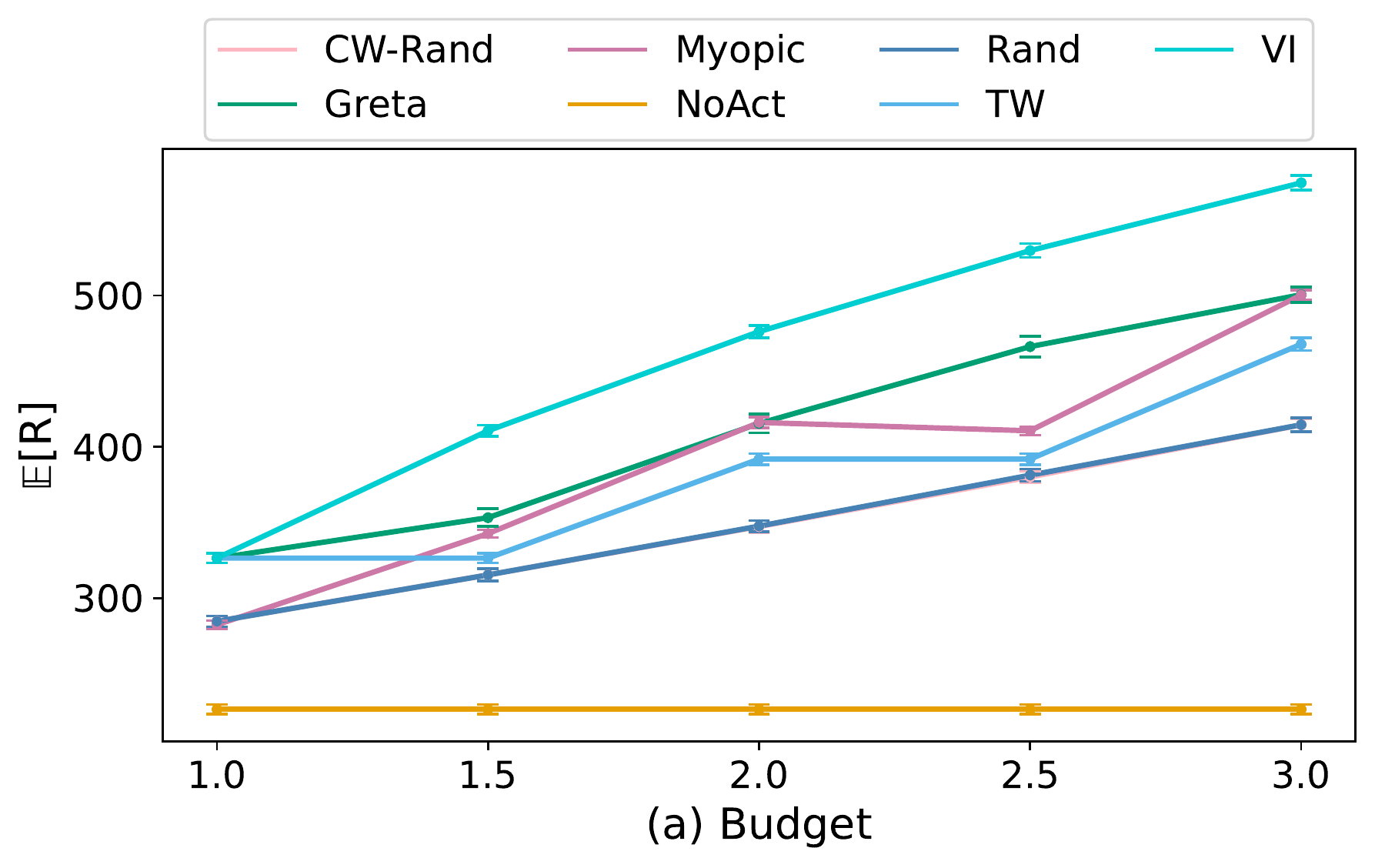}
  \caption{$\mathbb{E}[R]$ by policy and budget}
  \label{fig:gretaVsOpt}
\end{figure}

\subsection{\alg{} vs. Alternative Policies}
Here we compare \alg{} to the graph-agnostic and graph-aware comparison policies outlined in Section~\ref{sec:experimentalSetup}. We consider a synthetic cohort of $n=100$ restless arms whose transition matrices are randomly generated in such a way so as to satisfy the structural constraints introduced in Section~\ref{sec:model}. We use a stochastic block model (SBM) generator with $p_{in} = 0.2$ and $p_{out} = 0.05$, and consider both the \emph{random} and \emph{by cluster} options for $\upvarphi$. We let $T=120$, $B=10$, and $\psi = 0.5$.

In Table~\ref{tab:exp1}, we report results for each mapping-policy combination, along with margins of error for 95\% confidence intervals computed over 50 simulation seeds.

\newcolumntype{V}{!{\vrule width 1pt}}
\begin{table}[H]
\centering
\resizebox{0.7\columnwidth}{!}{%
\begin{tabular}{|c|l V l|}
	\specialrule{1pt}{1pt}{1pt}
	$\upvarphi(i)$& Policy & $\mathbb{E}[\textsc{IB}]$ (\%) ($\pm$) \\
	\specialrule{2.5pt}{1pt}{1pt}
	\multirow{6}{*}{randomly} 
	 &\textsc{Random} & 75.82 ~$\pm$ 0.890 \\
	 &\textsc{CWRandom} & 74.79 ~$\pm$ 1.068 \\
	 &\textsc{Myopic} & 87.83 ~$\pm$ 1.115 \\
	 &\textsc{TW} & 83.57 ~$\pm$ 0.779 \\
	 &\textsc{Greta}  & \textbf{100.00 ~$\pm$ 0.000 } \\	 
	 \specialrule{1.5pt}{1pt}{1pt}
	\multirow{6}{*}{by cluster}
	 &\textsc{Random} & 64.19 ~$\pm$ 0.786 \\
	 &\textsc{CWRandom} & 63.59 ~$\pm$ 0.804 \\
	 &\textsc{Myopic} & 76.24 ~$\pm$ 0.921 \\
	 &\textsc{TW} & 72.65 ~$\pm$ 0.684 \\
	 &\textsc{Greta}  & \textbf{100.00 ~$\pm$ 0.000 } \\
	\specialrule{1.5pt}{1pt}{1pt}
\end{tabular}%
}
\caption{$\mathbb{E}[\text{IB}]$ by choice of $\upvarphi$ and policy} \label{tab:exp1}
\end{table}
Key findings from this experiment include:
\begin{itemize}[leftmargin=1em, itemsep=0.1em, topsep=0ex]
    \item The policy produced by \alg{} achieves significantly higher $\mathbb{E}_\pi[\text{IB}]$ than each of the comparisons. 
    \item The gap in $\mathbb{E}_\pi[\text{IB}]$ between \alg{} and \textsc{Myopic}, which is the best-performing alternative, is
    larger for the \emph{by cluster} mapping than the \emph{random} mapping. This suggests that in assortative networks,
    relatively homogeneous transition dynamics within blocks facilitate the exploitation of the diminishing marginal costs associated with the pull-message dynamic.  
\end{itemize}

\subsection{Sensitivity Analysis}
We conduct sensitivity analysis with respect to: (1) the budget, $B$; (2) the message cost, $\psi$; and (3) the underlying graph topology, via the $p_{\text{in}}$ and $p_{\text{out}}$ hyperparameters of our stochastic block model graph generator. As we vary each of the aforementioned hyperparameters, we consider a fixed cohort size of $n=100$ randomly-generated, structural constraint-satisfying arms, a time horizon, $T=120$, and a mapping $\upvarphi: i \in [n] \mapsto v \in V$ from arms to vertices that is determined by cluster. We report unnormalized $\mathbb{E}_{\pi}[R]$, along with margins of error for 95\% confidence intervals computed over 50 simulation seeds, for \alg{}, \textsc{Threshold Whittle}, \textsc{NoAct}, and \textsc{Myopic}, which is the best-performing non-TW alternative. We describe each task below, and present results in Figure~\ref{fig:sensitivityResults}.

\begin{figure}[!h]
  \centering
  \includegraphics[width=0.92\linewidth]{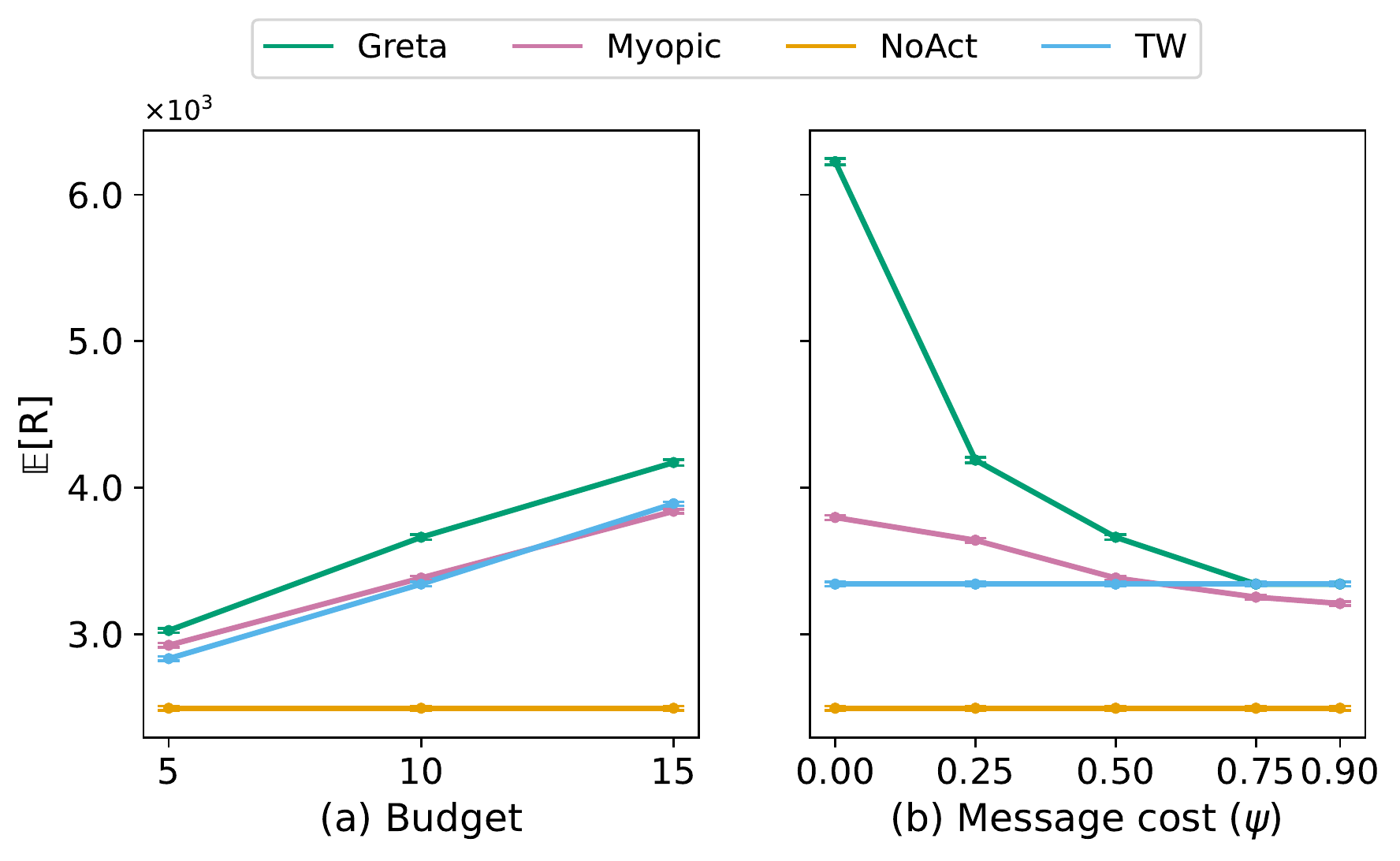}
  \label{fig:srBudgetMsgCost}
  \centering
  \includegraphics[width=0.92\linewidth]{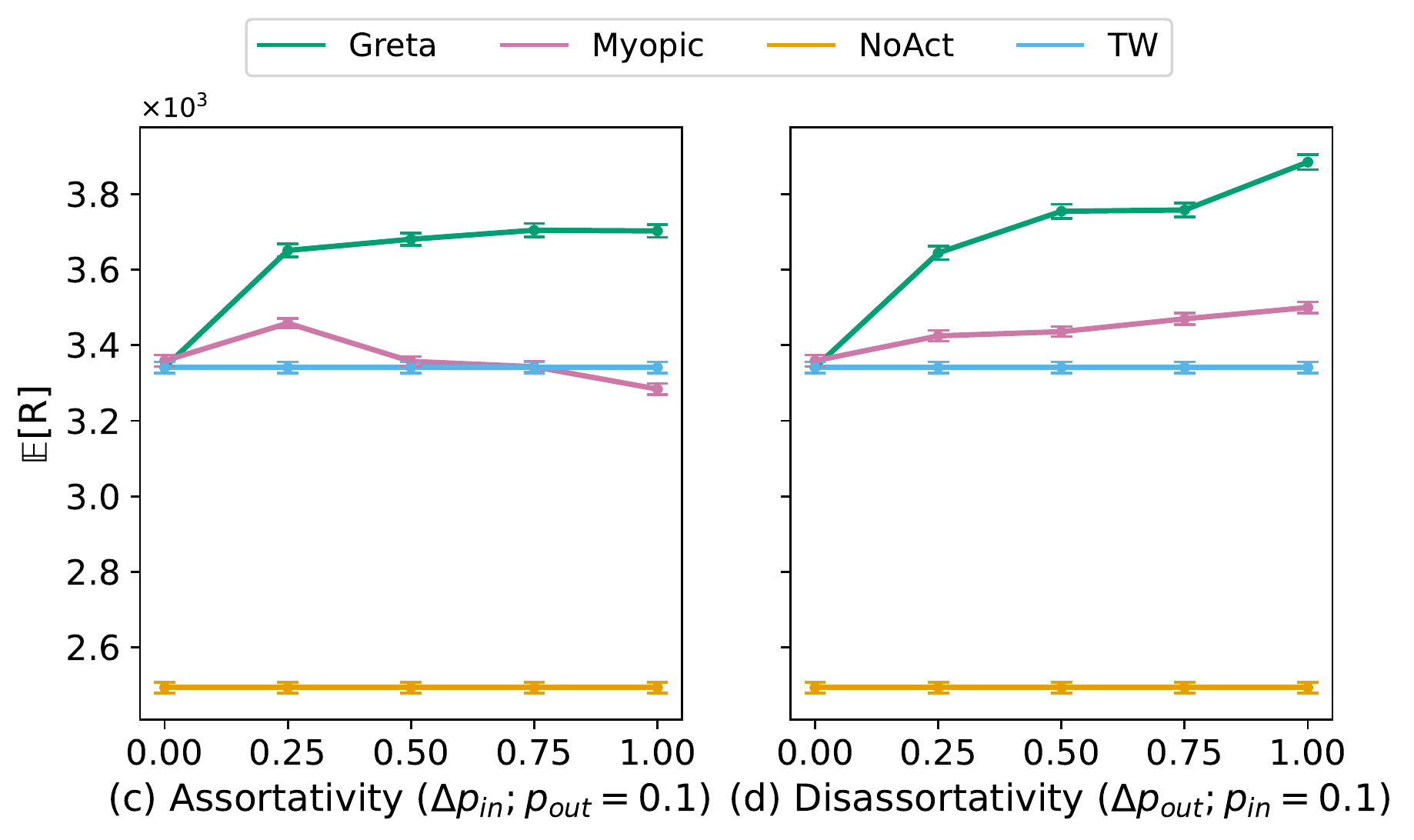}
  \label{fig:srAssortativity}
\caption{Sensitivity results, by varied hyperparameter}
\label{fig:sensitivityResults}
\end{figure}

\par \textbf{Budget}:
We hold message cost fixed at $\psi=0.5$, let $p_{\text{in}} = 0.25$, $p_{\text{out}}=0.05$, and consider values of $B \in \{5\%, 10\%, 15\%\} \text{ of }~n$. As Figure~\ref{fig:sensitivityResults}(a) illustrates, $\mathbb{E}_\pi[R]$ intuitively rises with $B$ for each policy considered. For each value of $B$, we find that \alg{} achieves higher $\mathbb{E}_\pi[R]$ than the comparison policies and that the gap between \alg{} and the best-performing alternative also increases with $B$.

\par \textbf{Message cost}:
Here, we hold the budget fixed at $6$, let $p_{\text{in}}=0.25$, $p_{\text{out}}=0.05$, and consider values of $\psi~\in~\{0.0, 0.25, 0.5, 0.75, 0.9\}$. As Figure~\ref{fig:sensitivityResults}(b) illustrates, $\mathbb{E}_\pi[R]$ decreases as the message cost, $\psi$, increases for \alg{} and \textsc{Myopic}, while it remains constant for active-action agnostic \textsc{NoAct} and message-agnostic \textsc{TW}. For each value of $\psi$ that we consider, \alg{} achieves higher $\mathbb{E}_\pi[R]$ than each of the comparison policies. This gap is intuitively largest when $\psi=0$, and decreases until \alg{} converges with \textsc{TW}---notably, without suffering loss in total expected reward due to divisibility issues with respect to $B$, when $\psi=0.75$. 

\par \textbf{Graph topology}:
We hold the budget fixed at $B=10$, let message cost, $\psi=0.5$, and consider two sets of increasingly \emph{assortative} (\emph{disassortative}) $(p_{\text{in}}$, $p_{\text{out}})$ ordered pairs. In each case, we start with $E=\emptyset$---i.e., $(0.0,0.0)$, and then hold $p_{\text{out}}$ ($p_{\text{in}}$) fixed at $0.1$ and steadily increase $p_{\text{in}}$ ($p_{\text{out}}$). Figure~\ref{fig:sensitivityResults}(c) and (d) present results. For \alg{}, while $\mathbb{E}_\pi[R]$ is generally increasing in the number of edges, the rate of growth levels off as assortativity rises but remains robust as disassortativity rises. This suggests that homophilic clustering of arms with respect to transition dynamics may undermine total welfare by inducing  competition within neighborhoods, while heterophilic clustering can help to smooth out subgroups' relative demand for constrained resources over time.

\section{Conclusion \& Future Work}
\label{sec:conclusion}
In this paper, we introduce networked restless bandits, a novel multi-armed bandit setting in which arms are restless and embedded in a directed graph. We show that this framework can be used to model constrained resource allocation in community settings, where receipt of the resource by an individual can result in spillover effects that benefit their neighbor(s). We also present \alg{}, a graph-aware, Whittle-based heuristic algorithm which is constrained reward-maximizing and budget-constraint satisfying in our networked restless bandit setting. Our empirical results demonstrate that the policy produced by \alg{} outperforms a set of graph-agnostic and graph-aware comparison policies for a range of different budgets, message costs, and graph topologies. Future directions include: (1) relaxing the assumption of perfect observability of transition matrices and/or graph topology; (2) considering individual and/or group fairness; and (3) incorporating sampling and/or distributed methods to improve scalability.

\appendix
\section{Heuristic Policy}
In this section, we provide additional details related to the heuristic policy we introduce in Section~\ref{sec:algorithm}, including pseudocode for subroutines and complete proofs for theorems.

\subsection{Pseudocode}
\label{app:pseudocode}

\subsubsection{Construct $G^\prime$}
\label{app:constructGprime}
Recall from Section~\ref{sec:algDetails} that \alg{} takes as input a pre-constructed augmented graph, $G^\prime = (V^\prime, E^\prime)$, where $V^\prime$ and $E^\prime$ represent augmented vertex and edge sets, respectively (Alg.~\ref{alg:heuristicClean}, line 1). Algorithm~\ref{alg:constructGprime} provides pseudocode for the construction of this augmented graph:
\begin{algorithm}[H]
\caption{Construct $G^\prime$}
\label{alg:constructGprime}
\footnotesize
\begin{algorithmic}[1] 
\Procedure{\textsc{Construct}}{$V, E$}
\State $V^\prime \gets V \cup \{-1\}$
\State $E^\prime \gets E \cup \{(u,-1) | u \in V\}$
\State $G^\prime = (V^\prime, E^\prime)$
\State\Return $V^\prime, E^\prime, G^\prime$
\EndProcedure
\end{algorithmic}
\end{algorithm}

The figure below visually represents this process for an example graph, $G$, and corresponding augmented graph, $G^\prime$:
\begin{figure}[h]
\centering 
\label{fig:constructGprime}
\includegraphics[scale=0.6]{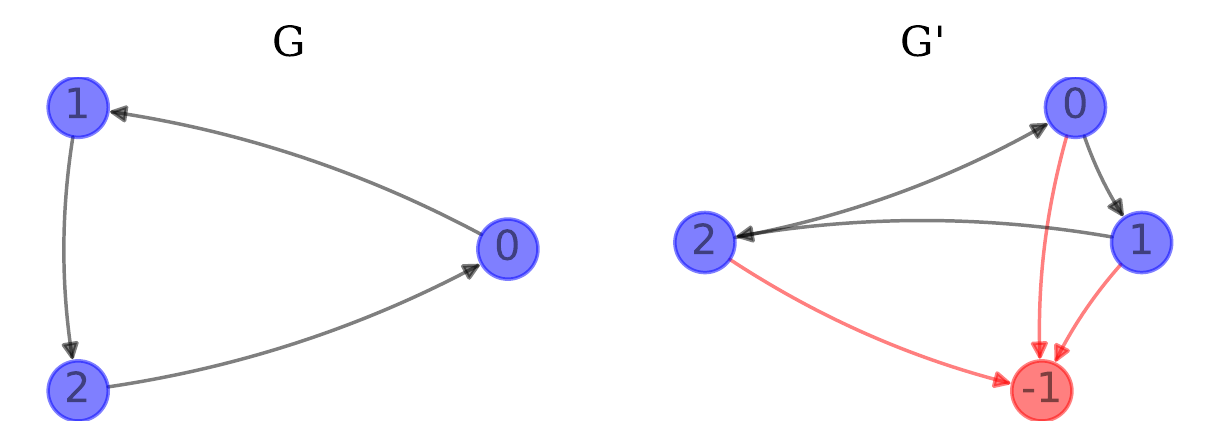}
\end{figure}

\subsubsection{\textsc{PullOnly} subroutine} 
Here we present pseudocode for the \textsc{PullOnly} subroutine called in line 6 of \alg{} (Algorithm~\ref{alg:heuristicClean}).

\begin{algorithm}[H]
\caption{Cumulative subsidy of top-2 arms (pulls)\\ Note: all sorts are descending; arrays are zero-indexed.}
\label{alg:pullsOnly}
\footnotesize
\begin{algorithmic}[1] 
\Procedure{\textsc{PullOnly}}{$E^\prime,\texttt{b} \in \mathbb{N}, W_2$}
\State $V^\prime_{2} \gets \{u \ | \ (u,v) \in E^\prime \land v = -1\}$
\State $\texttt{b} \gets \min(|V^\prime_{2}|, \texttt{b})$
\State $\texttt{pull\_vals} \gets \textsc{sort}(g: u \in V^\prime_{2} \mapsto (u,W_2^u))$
\State $\hat{a}_2: u \in \pi_\ell(\texttt{pull\_vals[:b]}) \mapsto 2$ 
\State $\nu_2 \gets \sum_{i=0}^{\texttt{b}}\pi_r(\texttt{pull\_vals}_i)$\\
\algLC{1}{Return top-\texttt{b} arm-actions and their cumulative subsidy}
\State\Return $\hat{a}_2, \nu_2$
\EndProcedure
\end{algorithmic}
\end{algorithm}

\subsubsection{\textsc{ModActsB} subroutine}
\label{app:updateActionsB}
Here we present pseudocode for the \textsc{ModActsB} subroutine called in line 9 or 12 (depending on the \textsc{if/else} in line 8) of \alg{} (Algorithm~\ref{alg:heuristicClean}), as well as in line 18 of the \textsc{MP} subroutine (Algorithm~\ref{alg:msgPull}).
\begin{algorithm}[!h]
\caption{Update $\vec{a}_t$ and budget per max-value action(s)}
\label{alg:updateActionsBudget}
\footnotesize
\begin{algorithmic}[1] 
\Procedure{\textsc{ModActsB}}{$G^\prime, C, \hat{a}^*, \vec{a}_t, B^\prime$}
\algLC{1}{For every vertex $u$ with an updated action:}
\For {$u \in \dom(\hat{a}^*)$}
\If {$\hat{a}^*(u) = 1$} 
    \algLC{4}{New action is a message; update $\vec{a}_t, B^\prime$}
    \State $B^\prime \mathrel{{-}{=}} C(1)(1 - \mathds{1}(\hat{a}^*(u) =1 \lor u = -1))$
    \State $\vec{a}_t^u \gets 1$
\ElsIf {$\hat{a}^*(u) = 2$} 
    \algLC{4}{New action is a pull; update $\vec{a}_t, B^\prime$}
    \State $B^\prime \mathrel{{-}{=}}  C(2)(1-\mathds{1}(\hat{a}^*(u) > 0)) + $ \\ \hspace{2.3cm} $\mathds{1}(\hat{a}^*(u) = 1)(C(2) - C(1))$
    \State $\vec{a}_t^u \gets 2$
    \algLC{4}{If $C(1)=0$, message every $v \in \calNout^{\prime}(u)$}
    \If {$\psi = 0$}
        \For {$v \in \calNout^{\prime}(u)$}
            \State $\vec{a}_t^v \gets \max(\vec{a}_t^v, 1)$
        \EndFor
    \EndIf
\EndIf
\EndFor 
\State\Return $\vec{a}_t, B^\prime$
\EndProcedure
\end{algorithmic}
\end{algorithm}

\subsubsection{\textsc{UpdateG} subroutine}
\begin{algorithm}[H]
\caption{Update $E^\prime \text{ and } G^\prime$}
\label{alg:updateGprime}
\footnotesize
\begin{algorithmic}[1] 
\Procedure{\textsc{UpdateGraph}}{$V^\prime, E^\prime, \hat{a}^*, E^\prime_{\oslash}$}
\For {$u \in \dom(\hat{a}^*)$}
\If {$\hat{a}^*(u) = 2$}
\State $E^\prime_{\oslash} \gets E^\prime_{\oslash} \cup \{e_{s,t} | e_{s,t} \in E^\prime \land$ \\\hspace{3.9cm} $\left(t=u\lor (s=u\land t=-1)\right)\}$
\EndIf 
\EndFor
\State $E^\prime \gets E^\prime \setminus  E^\prime_{\oslash}$
\State\Return $E^\prime, G^\prime=(V^\prime, E^\prime)$
\EndProcedure
\end{algorithmic}
\end{algorithm}

\subsection{Proofs}
\label{app:proofs}

\subsubsection{Bounding expected reward}
Here, we begin with Remark~\ref{remark:expLocalRNonDecinA}, and proceed to prove Theorem~\ref{thm:expectedReward} (restated for convenience below).\\
\begin{remark}
\label{remark:expLocalRNonDecinA}
For any restless arm $i \in [n]$ whose transition matrix entries satisfy the structural constraints introduced in Section~\ref{sec:networkedRmab}, action space $\mathcal{A}:=\{0:\texttt{no-act},1:\texttt{message},2:\texttt{pull}\}$, and non-decreasing local reward function, $r:~\mathcal{S}~\rightarrow~\mathbb{R}$, we have: $$\forall(a, a^\prime) \in \mathcal{A}\times\mathcal{A}, a < a^\prime \rightarrow \mathbb{E}[r^i_t | s^i_t, a_t] < \mathbb{E}[r^i_t | s^i_t, a^\prime_t]$$
\end{remark}

\expectedreward*

\begin{proof}
The \emph{first inequality}, $\mathbb{E}_{\textsc{TW}}[R] \leq \mathbb{E}_{\textsc{GH}, \psi > 0}[R]$, follows directly from how \textsc{Greta} constructs each $\vec{a}_t$. \textsc{Greta}'s \textsc{PullOnly} subroutine (Algorithm~\ref{alg:pullsOnly}) computes candidate action vector $\hat{a}_2$, which represents the actions we would take when following graph-agnostic \textsc{TW} for each budget chunk, $\texttt{b} \gets \min(B^{\prime}, 2)$. 

Since we do not execute $\vec{a}_t$ until we are finished constructing it---i.e., until we have run out of budget or edges, per the \textsc{while}-loop termination logic in line 8 of Algorithm~\ref{alg:heuristicClean}, the arms' states don't change from while loop iteration $i$ to $i+1$. Thus, per the inequality on line 11, \textsc{Greta} will return the same mapping from arms to actions as that returned by graph-agnostic \textsc{TW} with $k=\lfloor{B\rfloor}$ if $\bigwedge_i \mathds{1}(\nu_2^{i, (\cdot)} \geq \nu_{(1,2)}^{i, (\cdot)})$, and a mapping containing pulls and at least one message otherwise---i.e., when it is strictly advantageous to do so.

The \emph{second inequality}, $\mathbb{E}_{\textsc{GH}, \psi > 0}[R] \leq \mathbb{E}_{\textsc{GH}, \psi = 0}[R]$, follows directly  from Remark~\ref{remark:expLocalRNonDecinA} and the fact that when $\psi=0$, we can message at least as many arms as we can when $\psi>0$. Let $\mathcal{M}^u_t$ represent the set of \emph{potential} messages that is ``unlocked'' when a given node $u$ receives a pull at time $t$. The cardinality of $\mathcal{M}^u_t$ can be bound as follows:
\begin{equation}
    |\mathcal{M}^u_t| = \begin{cases}
			|\calNout^{\prime}(u)|, & \text{if $\psi = 0$}\\
            \min(\lfloor{\nicefrac{\texttt{b}}{\psi}\rfloor}, |\calNout^{\prime}(u)|), & \psi > 0
		 \end{cases}
\end{equation}

This is because when $\psi=0$, for each node $u$ that we pull at time $t$, we message \emph{every} one-hop outdegree neighbor $v \in \calNout^{\prime}(u)$ that has not already received a pull or message at time $t$ (see Algorithm~\ref{alg:updateActionsBudget}, lines 10-12). In contrast, when $\psi >0$, we must take the budget constraint into account.

The original claim follows from the transitive property. 
\end{proof}

\subsubsection{Computational complexity}
Here, we prove that \alg{} is efficiently computable in time polynomial in its inputs. We begin by introducing several lemmas related to the runtime of \alg's subroutines, and then proceed to prove  Theorem~\ref{thm:timeComplexity} (restated for convenience below, immediately preceding its corresponding proof).\\

\begin{remark}
\label{remark:construct}
The \textsc{Construct} subroutine (Algorithm~\ref{alg:constructGprime}) used to construct the augmented graph, $G^\prime$ , has computational complexity $O(|V|)$, since we insert the dummy vertex, $-1$, and a directed edge $(u,-1)$ for each $u \in V$.
\end{remark}

\begin{lemma}
\label{lemma:thresholdWhittle}
Using graph-agnostic \textsc{Threshold Whittle} to pre-compute the Whittle index for each vertex-active action combination $(v,\alpha) \in V^\prime \times \mathcal{A} \setminus \{0\}$ has time complexity $O(|V||\mathcal{A}||\mathcal{S}|^2T)$. 
\end{lemma}
\begin{proof}
Per \citet{mate2020collapsing}, \textsc{Threshold Whittle} has time complexity $O(|S|^2T)$ per arm-active action. There are $N = |V|$ arms and $|\mathcal{A} \setminus \{0\}|$ active actions. Thus, Algorithm~\ref{alg:whittle} has time complexity $O(|V||\mathcal{A}||\mathcal{S}|^2T)$.
\end{proof}

\begin{lemma}
\label{lemma:pullOnly}
The \textsc{PullOnly} subroutine (Algorithm~\ref{alg:pullsOnly}) has time complexity $O(|V^\prime| \log |V^\prime|)$.
\end{lemma}
\begin{proof}
Line 2 of Algorithm~\ref{alg:pullsOnly} has cost $O(|V^\prime|)$. Lines 3, 5, and 6 have cost $O(1)$. The cost of this subroutine is thus dominated by the cost of sorting the Whittle index values (line 4), which is $O(|V^\prime| \log |V^\prime|)$.
\end{proof}

\begin{lemma} 
\label{lemma:edgeIndices}
The \textsc{EdgeIndices} subroutine (Algorithm~\ref{alg:computeIndices}) has time complexity $O(|V^\prime| \log |V^\prime|)$.
\end{lemma}
\begin{proof}
Lines 2 and 4 of Algorithm~\ref{alg:computeIndices} have time complexity $O(1)$. The \textsc{for}-loop on lines 5-9 has time complexity $O(|V^\prime|)$, since in the worst case, $|\calNout^{\prime}(u)| = |V^\prime|$. Thus, the time complexity of this subroutine is dominated by the $O(|V^\prime| \log |V^\prime|)$ cost of sorting the Whittle index values (line 3).
\end{proof}

\begin{lemma}
\label{lemma:modActsB}
For $\psi \in [0,1)$, the \textsc{ModActsB} subroutine (Algorithm~\ref{alg:updateActionsBudget}) has time complexity $O(|V^\prime|)$ if $\psi > 0$ and $O(|V^\prime|^2)$ otherwise.  
\end{lemma}
\begin{proof}
In the \textsc{for}-loop contained in lines 2-12 of Algorithm~\ref{alg:updateActionsBudget}, we iterate over each vertex $u$ with an updated action (i.e., $\forall u \in \dom{\hat{a}^*}$), with cost $O(|V^\prime|)$. In lines 3-9, we decrement our remaining budget, $B^\prime$, and update our candidate action vector; each of these operations have time complexity $O(1)$. If $\psi >0$, lines 10-12 are not executed. Conversely, if $\psi = 0$, lines 10-12 \emph{are} executed, and we iterate over $u$'s one-hop, outdegree neighbors with cost $O(|V^\prime|)$. Thus, the time complexity of this subroutine is $O(|V^\prime|)$ if $\psi > 0$ and $O(|V^\prime|^2)$ otherwise. 
\end{proof}

\begin{lemma}
\label{lemma:updateGraph}
The \textsc{UpdateGraph} subroutine  (Algorithm~\ref{alg:updateGprime}) has time complexity $O(|V^\prime||E^\prime|)$.
\end{lemma}
\begin{proof}
In the \textsc{for}-loop contained in lines 2-5 of Algorithm~\ref{alg:updateGprime}, we iterate over each vertex $u$ with an updated action (i.e., $\forall u \in \dom{\hat{a}^*}$). Lines 4-5 and 6 each have cost $O(|E^\prime|)$. Thus, the time complexity of this subroutine is $O(|V^\prime||E^\prime|)$. 
\end{proof}

\begin{lemma}
\label{lemma:msgPullWhileLoop}
Let $\mathcal{W}_{MP}$ represent the number of \textsc{while}-loop iterations that occur during any given call to the \textsc{MsgPull} subroutine (i.e., Algorithm~\ref{alg:msgPull}, lines 7-22). For $\psi \in [0,1)$, we can upper-bound $\mathcal{W}_{MP}$ as follows:
\begin{equation*}
    \begin{cases} 
		O(\min(|E^{\prime}|, \lfloor{\frac{\texttt{b=2}}{\psi}\rfloor})), & \text{if $\psi > 0$}\\
        O(|V^{\prime}|), & \text{otherwise}
		 \end{cases}
\end{equation*}
\end{lemma}
\begin{proof}
There are two cases: $\psi >0$ and $\psi =0$. We consider each below:
\begin{case} $\psi > 0$: For any given arm $i \in [n]$, and \textsc{while}-loop iteration $j > 0$, let an \emph{action upgrade} represent a modification to $\hat{a}_{(1,2)}$ such that $\hat{a}_{(1,2)}^{(i,j)} > \hat{a}_{(1,2)}^{(i,j-1)}$. The small-\texttt{b} budget is initialized $=2$, and is strictly decreasing in the number of action upgrades when $\psi >0$. The worst case from a time complexity perspective will be when our only upgrades are messages, since $C(\texttt{pull}) > \psi \text{ for } \psi \in [0,1)$---i.e., $\lfloor{\frac{\texttt{b=2}}{\psi}\rfloor}$ iterations.
    
Additionally, at each iteration we either: (a) have at least one cost-feasible action upgrade, which results in the removal of at least one edge from $E^{\dprime}$ when we call \textsc{UpdateG} in line 19 (note: $E^{\dprime}$ is initialized in line 2 as a copy of $E^\prime$, so it is always the case that $|E^{\dprime}| \leq |E^\prime|$); or (b) do not have any cost-feasible action upgrades remaining, in which case, the \textsc{while}-loop terminates when we call \textsc{GetCost} in line 7. Thus, we can conduct at most $\min(|E^{\prime}|, \lfloor{\frac{\texttt{b=2}}{\psi}\rfloor})$ iterations before breaking. 
\end{case}

\begin{case} $\psi = 0$: Each time we select to pull vertex $u$ in lines 14-22, we message \emph{every} vertex $v \in \calNout^{\prime}(u)$ that is currently slated to receive a \texttt{no-act} (i.e., when we call the \textsc{ModActsB} subroutine on line 18). The next time we encounter $u$ in the \textsc{for}-loop on line 8, the construction of $\calNout^{\prime}(u)$ on line 9 will yield $\emptyset$, since we include only neighbors for which $\hat{a}_{(1,2)} = 0$. This will hold for each vertex $\in V^{\dprime}$, so we will break on line~13 after at most $O(|V^{\prime}|)$ iterations of the \textsc{while}-loop.
\end{case}
\end{proof}

\begin{lemma}
\label{lemma:msgPull}For notational convenience, let $\xi = \mathds{1}(\psi>0) \times \min(|E^{\prime}|^2, \lfloor{\frac{\texttt{b=2}}{\psi}\rfloor}|E^\prime|) + \mathds{1}(\psi=0) \times|V^{\prime}||E^{\prime}|$.
For $\psi \in [0,1)$, the \textsc{MsgPull} subroutine (Algorithm~\ref{alg:msgPull}) has time complexity $O\left(\max\left(\xi|V^\prime|^2\log|V^\prime|, \  \xi|V^\prime||E^\prime|^2 \right) \right)$ if $\psi > 0$, and $O\left(\max\left(\xi|V^\prime|^2\log|V^\prime|, \ \xi|V^\prime||E^\prime|^2, \ \xi|E^\prime||V^\prime|^2  \right) \right)$ otherwise.
\end{lemma}

\begin{proof}
Regardless of the value of $\psi$, lines 2-6 of Algorithm~\ref{alg:msgPull} are dominated by the cost of constructing $G^{\dprime}$ (line 2), which has cost $O(|V^\prime| + |E^\prime|)$. The \textsc{while}-loop termination check we perform in line 7 has time complexity $O(|E^{\prime}|)$ since we call \textsc{GetCost} (Algorithm~\ref{alg:getCost}) for each edge $e \in E^{\dprime}$. 

In the \textsc{for}-loop contained in lines 8-10, we loop over each vertex $v \in V^{\dprime}$ (line 8; $O\left(|V^\prime|\right)$), collect the one-hop, outdegree neighbors currently slated to receive a \texttt{no-act} (line 9; $O\left(|E^\prime|\right)$), and then call \textsc{EdgeIndices}. Thus, per Lemma~\ref{lemma:edgeIndices}, the time complexity of lines 8-10 is $O\left(\max\left(|V^\prime|^2 \log |V^\prime|, \ |V^\prime||E^\prime|\right)\right)$. 

The cost to sort the edge index values in line 11 is $O(|E^\prime| \log |E^\prime|)$; the termination check on line 12 has cost $O(1)$. 

In the \textsc{for}-loop contained in lines 14-22, we iterate over edge index values (line 14; $O(|E^\prime|)$). Lines 15-16 and 20-21 each have cost $O(1)$. Per Lemma~\ref{lemma:modActsB}, the call to \textsc{ModActsB} (Algorithm~\ref{alg:updateActionsBudget}) in line 18 has cost $O(|V^\prime|)$ if $\psi > 0$ and $O(|V^\prime|^2)$ otherwise. Per Lemma~\ref{lemma:updateGraph}, the call to \textsc{UpdateG} (Algorithm~\ref{alg:updateGprime}) in line 19 has time complexity $O(|V^\prime||E^\prime|)$. Thus, for $\psi \in [0,1)$, lines 14-22 have time complexity $O(|V^\prime||E^\prime|^2)$ if $\psi > 0$ and $O\left(\max\left(|V^\prime||E^\prime|^2, \ |V^\prime|^2|E^\prime|\right)\right)$ otherwise.

Finally, we consider the subroutine in its entirety. For notational convenience, let $\xi$ represent the outer cost of the \textsc{while}-loop, where the maximum number of iterations is defined per Lemma~\ref{lemma:msgPullWhileLoop}: \begin{small}$$
       \xi = \mathds{1}(\psi > 0) \times \min(|E^{\prime}|^2, \lfloor{\frac{\texttt{b=2}}{\psi}\rfloor}|E^\prime|) + 
       \mathds{1}(\psi = 0) \times~|V^{\prime}||E^\prime|$$ \end{small}
Then, for $\psi \in [0,1)$, the time complexity of the \textsc{MsgPull} is:
\begin{small}
$
\begin{cases}
		O\left(\max\left(\xi|V^\prime|^2\log|V^\prime|, \  \xi|V^\prime||E^\prime|^2 \right) \right), & \text{$\psi > 0$}\\
            O\left(\max\left(\xi|V^\prime|^2\log|V^\prime|, \ \xi|V^\prime||E^\prime|^2, \ \xi|V^\prime|^2|E^\prime|  \right) \right), & \text{else}
		 \end{cases}
$\end{small}
\end{proof}

\timecomplexity*

\begin{proof}
Per Remark~\ref{remark:construct} and Lemma~\ref{lemma:thresholdWhittle}, the time complexity of constructing \alg's inputs is dominated by precomputation of the Whittle indices, which has cost $O(|V||\mathcal{A}||\mathcal{S}|^2T)$. Lines 2-3 of \alg{} (Algorithm~\ref{alg:heuristicClean}) have cost $O(1)$.

The \textsc{while}-loop termination check we perform in line 4 has time complexity $O(|E^{\prime}|)$ since we call \textsc{GetCost} (Algorithm~\ref{alg:getCost}) for each edge $e \in E^{\prime}$. Line 5 has cost $O(1)$. 

Inside the \textsc{while}-loop (i.e., Algorithm~\ref{alg:heuristicClean}, lines 4-13), the call to the \textsc{MsgPull} subroutine on line 7 contributes the dominating cost. Per Lemma~\ref{lemma:msgPull}, for $\psi \in [0,1)$, this subroutine has time complexity:

\begin{small}
\begin{equation*}
\begin{cases}
		O\left(\max\left(\xi|V^\prime|^2\log|V^\prime|, \  \xi|V^\prime||E^\prime|^2 \right) \right), & \text{$\psi > 0$}\\
            O\left(\max\left(\xi|V^\prime|^2\log|V^\prime|, \ \xi|V^\prime||E^\prime|^2, \ \xi|V^\prime|^2|E^\prime|  \right) \right), & \text{else}
		 \end{cases}
\end{equation*}
\end{small}
where \begin{small}$$\xi = \mathds{1}(\psi > 0) \times \min(|E^{\prime}|^2, \lfloor{\frac{\texttt{b=2}}{\psi}\rfloor}|E^\prime|) + \mathds{1}(\psi = 0) \times~|V^{\prime}||E^\prime|$$\end{small}

Next, we can bound the number of  \textsc{while}-loop iterations that occur during any given call to \alg{} in a way that proceeds identically to the bound we establish for the \textsc{MsgPull} subroutine's \textsc{while}-loop in Lemma~\ref{lemma:msgPullWhileLoop}, with one small modification: we replace $\texttt{b}$ with the full budget, $B$, noting that $\texttt{b} \leq B$. This accounts for the presence of the $\xi^2$ term in Theorem~\ref{thm:timeComplexity}. 

Finally, we note that we call \alg{} once per timestep $t$. Thus, when we consider the time complexity over time horizon, $T$, the cost of computing $\vec{a}_t$ $T$ times dominates the $O(|V||\mathcal{A}||\mathcal{S}|^2T)$ cost of precomputing the Whittle indices, and we are done.
\end{proof}

\section{Additional empirical results}

Here we provide additional empirical support for Theorem~\ref{thm:expectedReward}, reproduced below for convenience:
\expectedreward*

We note that for any given directed graph, $G=(V,E)$, for fixed $|V|$, as $|E|$ tends toward a complete graph (i.e., $|E|~\rightarrow~|V|*(|V|-1)$), as long as we have budget $B \geq C(2) + C(1)$ (or, equivalently, $B \geq 1 + \psi$), we can pull \emph{any} arm and gain the ability to message \emph{any other} arm. Thus, in the complete graph setting, we can modify the constrained optimization-based approach represented by Equation~\ref{eqn:objective} (reproduced for convenience below) to: (1) include the requirement that at least one arm receive a pull (i.e., constraint c); and (2) remove the neighborhood constraint (i.e., constraint a), since it is guaranteed to be satisfied without being explicitly enforced.

{\small
{\centering
\begin{equation*}
\label{eqn:objectiveRestate}
\begin{array}{l@{}ll}
J(\mathbf{s}) = \max\limits_{\mathbf{X}} & \left\{ \displaystyle\sum_{i=0}^{n-1} r^i(s^i) + \beta \mathbb{E}[J(\mathbf{s}^\prime), \mathbf{X}]\right\}  & \\

\text{subject to } & \displaystyle\sum_{i=0}^{n-1} \sum_{j=0}^{|\mathcal{A}|-1} x_{i,j} \cdot c_j \leq B & \\

(a) & x_{i,1} \leq \displaystyle\bigvee_{i^\prime \in \calNin(i)}
x_{i^\prime, 2} & \forall i \in [n]\\ 

(b) & \displaystyle\sum_{j = 0}^{|\mathcal{A}|-1} x_{i,j} = 1 & \forall i \in [n]\\

(c) & \displaystyle\sum_{i=0}^{n-1} x_{i,2} \geq 1 & \\
(d) & \mathbf{X} \in \{0,1\}^{n \times |\mathcal{A}|} & 

\end{array}
\end{equation*}}}\\

Thus, for fixed set of restless arms with cardinality $|V|$, and message cost, $\psi$, we can upper-bound $\mathbb{E}_{\textsc{GH}, \psi}[R]$ by the expected reward achieved by the modified math program when the graph in question is complete. Intuitively, this bound will become tighter as the cardinality of $|E|$ is increased.

To empirically validate this claim, we consider a synthetic cohort of $n=100$ restless arms whose transition matrices are randomly generated in such a way so as to satisfy the structural constraints introduced in Section~\ref{sec:model}. We let $T = 120, B=10$, and $\psi = 0.5$. We can then construct graphs using this fixed vertex set but containing edge sets with varying cardinalities, expressed as a percentage of the number of edges the complete graph would contain. 

More concretely, we define a set of six edge generation seeds, and use each seed to select subsets of edges uniformly at random, such that the subsets have cardinalities $\in \{0.0, 0.1, 0.2, 0.3, 0.4, 0.5, 0.6, 0.7, 0.8, 0.9, 1.0\}$ of the complete graph. For each graph so constructed, we report unnormalized $\mathbb{E}_\pi[R]$, along with margins of error for 95\% confidence intervals computed over 30 simulation iterations (see subfigures a-f). We observe that:
\begin{enumerate}
    \item For each seed, and every value of $|E|$, $\mathbb{E}_{\textsc{TW}, \psi}[R] \leq \mathbb{E}_{\textsc{GH}, \psi}[R] < \mathbb{E}_{\textsc{MP},\psi}[R]$.
    \item We note that while \textsc{Greta}'s expected reward does not monotonoically increase with $|E|$ in every case, this is to be expected, since we are not guaranteed to get the same subset of edges we had for smaller percentage values as we increase $|E|$. This being said, expected reward is generally increasing with $|E|$, and as $|E|~\rightarrow |V|(|V|-1)$, $\mathbb{E}_{\textsc{GH}, \psi}[R] \rightarrow \mathbb{E}_{\textsc{MP},\psi}[R]$.
    \end{enumerate}

\begin{figure}[!h]
\centering     
\subfigure[Edge generation seed 1]{\includegraphics[width=0.85\columnwidth]{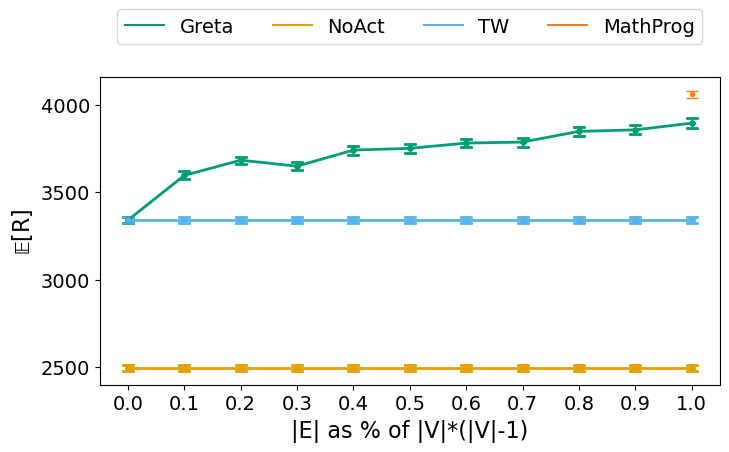}}
\subfigure[Edge generation seed 2]{\includegraphics[width=0.85\columnwidth]{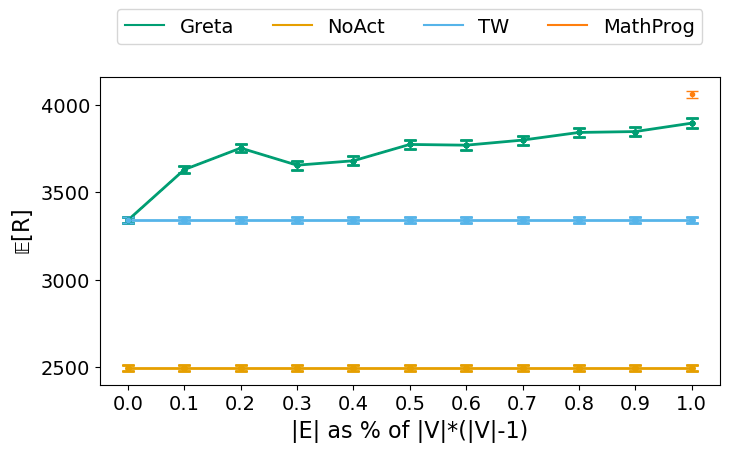}}
\subfigure[Edge generation seed 3]{\includegraphics[width=0.85\columnwidth]{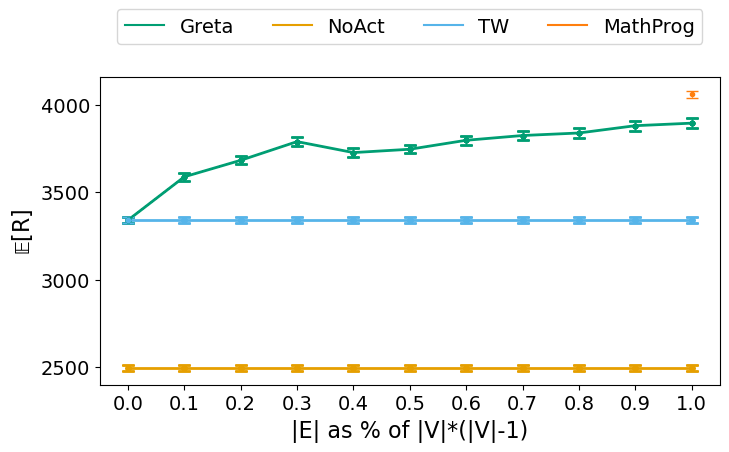}}
\end{figure}


\begin{figure}[!h]
\centering     
\subfigure[Edge generation seed 4]{\includegraphics[width=0.85\columnwidth]{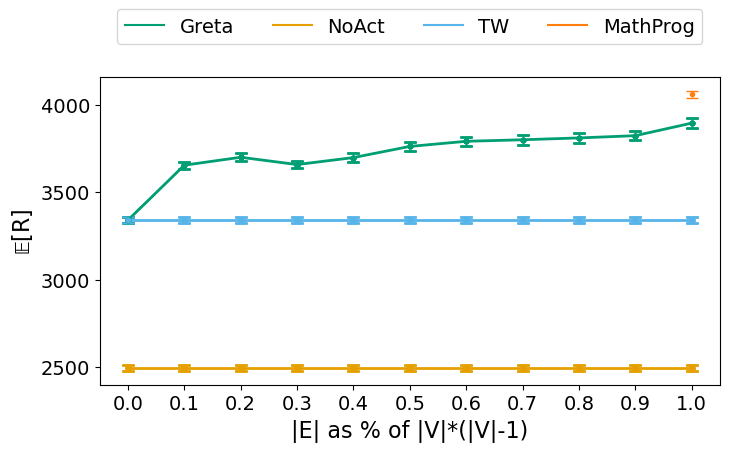}}
\subfigure[Edge generation seed 5]{\includegraphics[width=0.85\columnwidth]{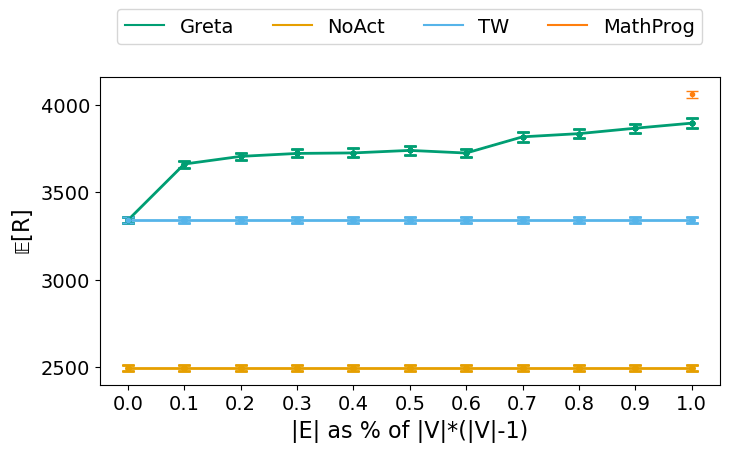}}
\subfigure[Edge generation seed 6]{\includegraphics[width=0.85\columnwidth]{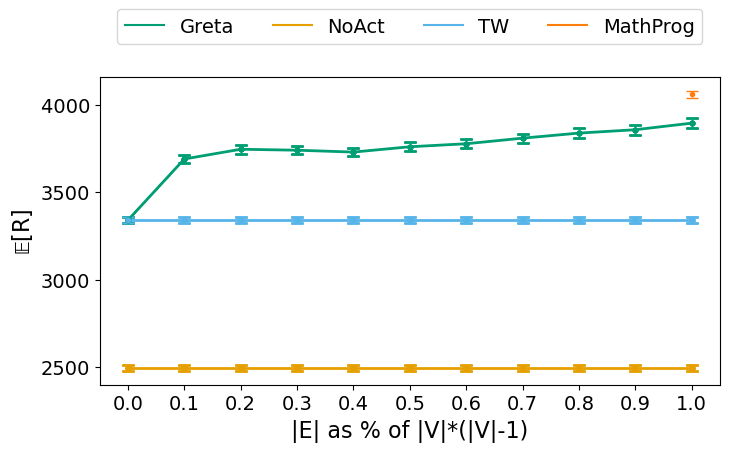}}
\end{figure}

\section{Additional experimental details}
\label{app:experimentDetails}
\textbf{Source code}: For the source code and instructions needed to reproduce the experimental results we report in  Section~\ref{sec:experimental}, see:
\href{https://github.com/crherlihy/networked_restless_bandits}{\faGithub{} \texttt{crherlihy/networked\_restless\_bandits}}.\\

\textbf{Compute resources}: We ran all of our simulations on a MacBook Pro with a 2 GHz Quad-Core Intel(R) Core i5 CPU and 16 GB of RAM.

\section*{Acknowledgments}
We were supported by NSF CAREER Award IIS-1846237, NSF D-ISN Award \#2039862, NSF Award CCF-1852352, NIH R01 Award NLM013039-01, NIST MSE Award \#20126334, DARPA GARD \#HR00112020007, DoD WHS Award \#HQ003420F0035, ARPA-E Award \#4334192, ARL Award W911NF2120076 and a Google Faculty Research Award. The views and conclusions contained in this publication
are those of the authors and should not be interpreted as representing official policies or endorsements of U.S. government or funding agencies. We thank Samuel Dooley, Pranav Goel, Aviva Prins, Dr. Philip Resnik, and our anonymous reviewers for their helpful input and feedback. 

\bibliography{main}
\end{document}